\documentclass{article}

\usepackage[preprint]{neurips_2026}

\usepackage[utf8]{inputenc}
\usepackage[T1]{fontenc}
\usepackage{hyperref}
\hypersetup{colorlinks=false, pdfborder={0 0 0}, hypertexnames=false,
  pdftitle={TreeWalker: Partial Evaluation for Grouped Tree-Ensemble Inference},
  pdfauthor={Durmus Karatay, Richard Newman}}
\usepackage{url}
\newcommand{\codeurl}{\url{https://github.com/Shopify/treewalker/tree/neurips2026-final}}
\usepackage{booktabs}
\usepackage{amsfonts}
\usepackage{amsmath,amssymb}
\usepackage{amsthm}
\theoremstyle{plain}
\newtheorem{theorem}{Theorem}
\newtheorem{lemma}{Lemma}
\newtheorem{corollary}{Corollary}
\theoremstyle{definition}
\newtheorem{definition}{Definition}
\usepackage{nicefrac}
\usepackage{graphicx}
\usepackage{xcolor}
\usepackage{algorithm}
\usepackage{algpseudocode}
\usepackage{tabularx}
\usepackage{natbib}
\usepackage{float} %
\usepackage{tikz}
\usetikzlibrary{arrows.meta}

\definecolor{twnavy}{HTML}{2D4254}
\definecolor{twcoral}{HTML}{E87A5D}
\definecolor{twgreen}{HTML}{8DB584}

\graphicspath{{figures/}}

\title{TreeWalker: Partial Evaluation\\for Grouped Tree-Ensemble Inference}

\author{%
  Durmus Karatay \\
  Shopify \\
  \texttt{durmus.karatay@shopify.com} \\
  \And
  Richard Newman \\
  Shopify \\
  \texttt{richard.newman@shopify.com} \\
}
\date{}

\begin{document}

\maketitle

\begin{abstract}
Many prediction tasks evaluate a trained tree ensemble on groups of
rows that share feature values: discrete-time survival models expand
each patient into $G$ time steps, click-through-rate models score every
item in a search session, and scenario analyses vary a few inputs while
holding the rest fixed. Standard inference treats each row
independently and ignores this input-side correlation.

We present TreeWalker, which instantiates partial evaluation for
grouped tree-ensemble inference: constant features are static, varying
features dynamic. TreeWalker walks each tree once per group, partitions
a row bitmask at varying splits, and skips empty subtrees via an
unsplit shortcut. It reads standard LightGBM and XGBoost models;
training is unchanged.

We prove a structural work decomposition: per-tree work splits into
the constant-projected subtree size $|T_c|$, $G$ leaf writes, and a
predicate-mask provisioning cost $Q$. For the trace evaluator, per-row
work approaches $(d_v+1)/(d+1)$ as $G \to \infty$, giving asymptotic
speedup $1/(1-\bar\rho_G)$.

On Intel, TreeWalker is ${\sim}3\times$ faster than a row-independent
traversal at the reference model size ($T{=}500$, $L{=}8$; $G{=}16$
for SUPPORT/FLCHAIN and empirical groups for Expedia) and
$7.7$--$9.1\times$ faster at $G{=}128$ on the survival datasets; on
Arm, $2.5$--$2.9\times$ and $5.5$--$6.0\times$. On a scenario-analysis
benchmark it is faster in all $16$ configurations on both machines.
For f64 models, outputs match treelite GTIL up to tree-ordering
roundoff; for f32 models, TreeWalker's f64 leaf accumulation is closer
than native f32 to an exactly rounded reference on every audited row.
\end{abstract}
\section{Introduction}\label{sec:introduction}

Gradient-boosted decision tree (GBDT) ensembles dominate tabular
prediction~\citep{chen2016xgboost}. In many applications the ensemble
scores a \emph{group} of rows that share most feature values, not a
single row (\S\ref{sec:grouped}).

\paragraph{What's missing.}
Fast tree-inference systems optimize either per-visit cost
(compilation~\citep{boehm2021lleaves, nakandala2020mlprediction},
vectorization~\citep{nvidia2024fil}, data
layout~\citep{asadi2014runtime}) or per-walk cost (early termination
via leaf-bitvector intersection~\citep{lucchese2015quickscorer}). None of these exploit
\emph{input-side} correlation across the rows of a group, even though $48$--$70\%$ of features are constant within a group in our three factorial-grid datasets. Each row pays the full cost of its own walk, repeating the
constant-feature prefix $G$ times.

\paragraph{Partial evaluation across rows.}
We instantiate online partial evaluation~\citep{jones1993partial,
futamura1999partial, wuerthinger2017partial} for grouped tree
inference: constant features are the static inputs, varying features
the dynamic, and the recursion the binding-time fold.
TreeWalker\footnote{Code and benchmark harness, at the version measured here: \codeurl.} walks each
tree once per group, maintaining a row bitmask
(\texttt{u16}/\texttt{u32}/\texttt{u64}, or an array of words for
wider groups). The constant prefix advances in $O(1)$ regardless of $G$. At
each varying split the algorithm partitions the bitmask, prunes empty
sides via an unsplit shortcut, and continues. To our knowledge no prior
work exploits constant-feature partitioning across batched rows for
tree inference; the closest related work (the QuickScorer family) optimizes
per-walk cost via leaf-bitvector intersection within each tree,
orthogonal to our axis.

\paragraph{Operational contract.}
TreeWalker targets workloads whose request semantics supply grouped
rows; it does not construct groups from unstructured batches, and it
offers no advantage when rows share no feature values. In the target
workloads the task itself creates the group: discrete-time expansion
materializes $G$ rows from one entity, ranking sessions arrive grouped,
and scenario analysis generates the variants. TreeWalker is therefore
not a drop-in replacement for unstructured batch inference.

\paragraph{Contributions.}
\begin{enumerate}
\item A grouped-input partial evaluation algorithm with constant-walk
and row-mask propagation (\S\ref{sec:algorithm}).
\item A structural work decomposition theorem with two corollaries: a
trace-evaluator asymptote of $1/(1-\bar\rho_G)$ and a forest-level
provisioning bound for the default implementation
(\S\ref{sec:analysis}).
\item Empirical validation of the decomposition's predicted asymptote,
a factorial-grid speedup study on three datasets against LightGBM,
XGBoost, tl2cgen, lleaves, and QuickScorer, and a scenario-analysis
benchmark on a fourth (\S\ref{sec:evaluation}).
\end{enumerate}
\section{Background}\label{sec:background}

\subsection{Grouped Prediction Tasks}\label{sec:grouped}

A \emph{grouped prediction task} evaluates a model on sets of rows that
share a common subset of features. Each observation $i$ produces a
group of $G_i$ input vectors $\{(x_c^{(i)}, x_v^{(i,g)})\}_{g=1}^{G_i}$,
where $x_c$ are \emph{constant features} (identical across rows) and
$x_v$ are \emph{varying features} (potentially distinct per row). Three concrete
instances motivate our work: discrete-time survival
models~\citep{singer1993time, lee2018deephit} expand each patient into
$G$ rows over a time horizon (covariates constant, time features
varying); session-level CTR models~\citep{expedia2013kaggle} score
each item in a search session (user/query constant, item attributes
varying); scenario analyses evaluate a row under multiple parameter
perturbations. The \emph{constant fraction} $|x_c| / (|x_c| + |x_v|)$ proxies
shareable work: $48\%$ (Expedia) to $70\%$ (SUPPORT) across our factorial-grid benchmarks. The realized work-share, captured by $(d_v+1)/L$ in
Table~\ref{tab:decomp-results}, depends on the model's split-feature distribution.

\paragraph{Grouped data in modeling.}
Multilevel and nested-data models~\citep{gelman2006multilevel}, tree
models with random effects for clustered and longitudinal
data~\citep{sela2012reem, hajjem2014merf, sigrist2022gpboost}, and GBDT
forecasters for large collections of time
series~\citep{makridakis2022m5} specify models \emph{for} grouped data.
TreeWalker is orthogonal to that choice: it accelerates exact
inference for tree ensembles on grouped rows, however they were
trained. These settings are sources of grouped workloads, not competing
methods.

\paragraph{Notation.}
$T$ denotes the number of trees, $L$ the maximum tree depth, $G$ the
group width (number of rows per observation), $d$ the average root-to-leaf depth, $f$ a feature index, and $r \in \{0,\dots,G-1\}$ the row index within a group.

\subsection{GBDT Inference}\label{sec:gbdt}

A gradient-boosted tree ensemble $F(x) = \sum_{t=1}^{T} h_t(x)$, where
$h_t$ is the $t$-th tree, predicts by traversing each tree from root to
leaf and summing the leaf outputs. Per-row cost is worst-case $O(TL)$,
with average cost tracking $T \cdot d$ for average root-to-leaf depth
$d \le L$; evaluating rows independently multiplies both by $G$ per
group.

\paragraph{Axes of fast tree inference.}
Three cost axes characterize fast tree inference; prior work occupies
the first two.

\emph{Per-visit cost.} Compilation lowers interpretation overhead at
each internal node: lleaves~\citep{boehm2021lleaves} compiles trees to
LLVM IR, tl2cgen~\citep{tl2cgen2024} generates C from
treelite~\citep{cho2018treelite} models, and
Hummingbird~\citep{nakandala2020mlprediction} reformulates trees as
tensor programs; Yggdrasil-DF~\citep{guillame2022tfdf} uses SIMD-aware
ahead-of-time compilation. Vectorization fuses per-visit work across
observations: RAPIDS~FIL~\citep{nvidia2024fil} and ONNX~Runtime's
\texttt{TreeEnsemble}~\citep{onnxruntime_tree} are representative.
Cache-aware data layouts~\citep{asadi2014runtime} reduce per-visit
memory cost.

\emph{Per-walk cost.} Early-termination methods amortize predicate
evaluation across the ensemble:
QuickScorer~\citep{lucchese2015quickscorer},
V-QuickScorer~\citep{lucchese2017vquickscorer}, and
RapidScorer~\citep{ye2018rapidscorer} use per-tree leaf-bitvectors,
intersected across trees, to short-circuit traversal. The only open-source implementation~\citep{quickscorer_rs}
limits trees to $\le 128$ leaves and has neither categorical splits
nor missing-value routing; it does not cover our $L \ge 8$ cells
or exploit grouped inputs (comparison on an encoded variant in
Appendix~\ref{sec:app_quickscorer}).

\emph{Number-of-walks (this work).} Partial evaluation against shared
constants reduces the count of tree walks from $GT$ toward $T$ by
fusing constant-feature prefixes across the rows of a group. The axis
is orthogonal to per-visit and per-walk optimizations and composes
with them (\S\ref{sec:discussion}).
\section{Algorithm}\label{sec:algorithm}

\subsection{Recursive Partial Evaluation}\label{sec:partialeval}

Rows sharing constant features follow the same path through the tree
until they encounter a varying split. TreeWalker exploits this by
walking each tree once per group, not once per row. The walk
maintains two pieces of state: a node index (position in the tree)
and a row bitmask (which of the $G$ rows are still active).
Algorithm~\ref{alg:partial_eval} states the conceptual recursion;
Figure~\ref{fig:worked_example} illustrates the idea on a toy tree
with a group of four rows.

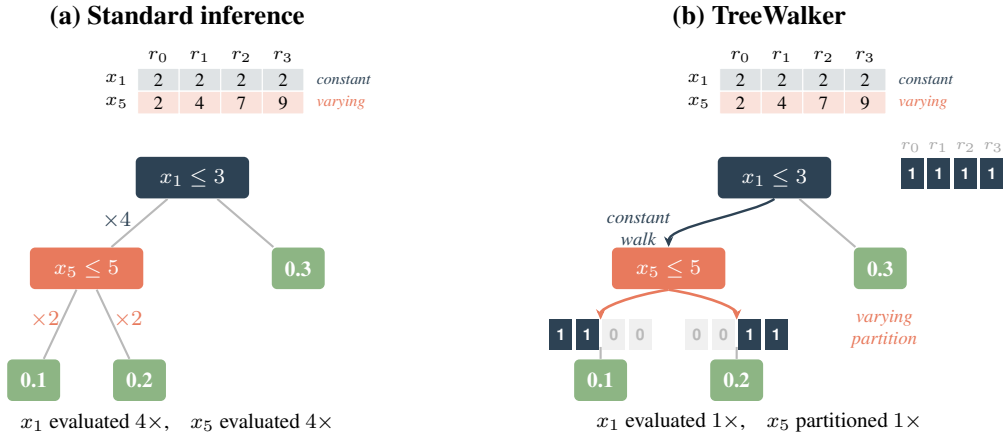
\begin{figure}[!htb]
\centering
\begin{tikzpicture}[
  y=0.8cm,
  font=\footnotesize,
  >={Stealth[length=4pt,width=3.5pt]},
  internal/.style={rectangle, rounded corners=2pt, fill=twnavy, text=white,
                   inner sep=3pt, minimum width=15mm, minimum height=5.5mm, align=center},
  vary/.style={rectangle, rounded corners=2pt, fill=twcoral, text=white,
               inner sep=3pt, minimum width=15mm, minimum height=5.5mm, align=center},
  leafbox/.style={rectangle, rounded corners=2pt, fill=twgreen, text=white,
                  font=\footnotesize\bfseries, inner sep=2pt,
                  minimum width=7mm, minimum height=5.5mm, align=center},
  tedge/.style={draw=gray!55, line width=0.8pt},
  costN/.style={text=twnavy, font=\footnotesize\bfseries},
  costC/.style={text=twcoral, font=\footnotesize\bfseries},
  annotN/.style={text=twnavy, font=\scriptsize\itshape, align=center},
  annotC/.style={text=twcoral, font=\scriptsize\itshape, align=center},
  mon/.style={rectangle, fill=twnavy, text=white, inner sep=0pt,
              minimum width=3mm, minimum height=4mm, font=\tiny\sffamily\bfseries},
  moff/.style={rectangle, fill=gray!12, text=gray!55, inner sep=0pt,
               minimum width=3mm, minimum height=4mm, font=\tiny\sffamily\bfseries},
  ptitle/.style={font=\bfseries},
  bottomcap/.style={font=\footnotesize, align=center},
  hdr/.style={inner sep=1.2pt, minimum width=5.5mm, minimum height=3.0mm, font=\scriptsize},
  navcell/.style={inner sep=1.2pt, minimum width=5.5mm, minimum height=3.0mm,
                  font=\scriptsize, fill=twnavy!15, draw=white, line width=0.4pt},
  varcell/.style={inner sep=1.2pt, minimum width=5.5mm, minimum height=3.0mm,
                  font=\scriptsize, fill=twcoral!22, draw=white, line width=0.4pt},
]

\begin{scope}[local bounding box=panelA]
  \node[ptitle] at (-0.2, 5.0) {(a)~Standard inference};

  \node[hdr]     at (-0.45, 4.35) {$r_0$};
  \node[hdr]     at ( 0.10, 4.35) {$r_1$};
  \node[hdr]     at ( 0.65, 4.35) {$r_2$};
  \node[hdr]     at ( 1.20, 4.35) {$r_3$};
  \node[hdr]     at (-1.00, 3.97) {$x_1$};
  \node[navcell] at (-0.45, 3.97) {2};
  \node[navcell] at ( 0.10, 3.97) {2};
  \node[navcell] at ( 0.65, 3.97) {2};
  \node[navcell] at ( 1.20, 3.97) {2};
  \node[hdr]     at (-1.00, 3.59) {$x_5$};
  \node[varcell] at (-0.45, 3.59) {2};
  \node[varcell] at ( 0.10, 3.59) {4};
  \node[varcell] at ( 0.65, 3.59) {7};
  \node[varcell] at ( 1.20, 3.59) {9};
  \node[text=twnavy,  font=\tiny\itshape, anchor=west] at (1.52, 3.97) {constant};
  \node[text=twcoral, font=\tiny\itshape, anchor=west] at (1.52, 3.59) {varying};

  \node[internal] (rA)  at ( 0.0,  2.35) {$x_1 \le 3$};
  \node[vary]     (lA)  at (-1.4,  0.85) {$x_5 \le 5$};
  \node[leafbox]  (rlA) at ( 1.4,  0.85) {0.3};
  \node[leafbox]  (llA) at (-2.1, -1.00) {0.1};
  \node[leafbox]  (lrA) at (-0.7, -1.00) {0.2};

  \draw[tedge] (rA) -- (lA);
  \draw[tedge] (rA) -- (rlA);
  \draw[tedge] (lA) -- (llA);
  \draw[tedge] (lA) -- (lrA);

  \node[costN] at (-1.00,  1.65) {$\times 4$};
  \node[costC] at (-1.95,  0.00) {$\times 2$};
  \node[costC] at (-0.85,  0.00) {$\times 2$};

  \node[bottomcap] at (-0.2, -1.70) {$x_1$ evaluated $4\times$,\quad $x_5$ evaluated $4\times$};
\end{scope}

\begin{scope}[xshift=7.7cm, local bounding box=panelB]
  \node[ptitle] at (-0.2, 5.0) {(b)~TreeWalker};

  \node[hdr]     at (-0.45, 4.35) {$r_0$};
  \node[hdr]     at ( 0.10, 4.35) {$r_1$};
  \node[hdr]     at ( 0.65, 4.35) {$r_2$};
  \node[hdr]     at ( 1.20, 4.35) {$r_3$};
  \node[hdr]     at (-1.00, 3.97) {$x_1$};
  \node[navcell] at (-0.45, 3.97) {2};
  \node[navcell] at ( 0.10, 3.97) {2};
  \node[navcell] at ( 0.65, 3.97) {2};
  \node[navcell] at ( 1.20, 3.97) {2};
  \node[hdr]     at (-1.00, 3.59) {$x_5$};
  \node[varcell] at (-0.45, 3.59) {2};
  \node[varcell] at ( 0.10, 3.59) {4};
  \node[varcell] at ( 0.65, 3.59) {7};
  \node[varcell] at ( 1.20, 3.59) {9};
  \node[text=twnavy,  font=\tiny\itshape, anchor=west] at (1.52, 3.97) {constant};
  \node[text=twcoral, font=\tiny\itshape, anchor=west] at (1.52, 3.59) {varying};

  \node[internal] (rB)  at ( 0.0,  2.35) {$x_1 \le 3$};
  \node[vary]     (lB)  at (-1.4,  0.85) {$x_5 \le 5$};
  \node[leafbox]  (rlB) at ( 1.4,  0.85) {0.3};
  \node[leafbox]  (llB) at (-2.3, -1.00) {0.1};
  \node[leafbox]  (lrB) at (-0.5, -1.00) {0.2};

  \draw[tedge] (rB) -- (rlB);

  \node[font=\tiny, text=gray!70] at ( 2.35,  2.85) {$r_0$\ \ $r_1$\ \ $r_2$\ \ $r_3$};
  \node[mon] at ( 1.825, 2.42) {1};
  \node[mon] at ( 2.175, 2.42) {1};
  \node[mon] at ( 2.525, 2.42) {1};
  \node[mon] at ( 2.875, 2.42) {1};

  \draw[->, twnavy, line width=0.9pt]
    (rB.south) .. controls (-0.30, 1.75) and (-1.40, 1.55) .. (lB.north);
  \node[annotN] at (-1.80, 1.50) {constant\\walk};

  \node[mon]  at (-2.825, -0.25) {1};
  \node[mon]  at (-2.475, -0.25) {1};
  \node[moff] at (-2.125, -0.25) {0};
  \node[moff] at (-1.775, -0.25) {0};
  \node[moff] at (-1.025, -0.25) {0};
  \node[moff] at (-0.675, -0.25) {0};
  \node[mon]  at (-0.325, -0.25) {1};
  \node[mon]  at ( 0.025, -0.25) {1};

  \draw[->, twcoral, line width=1pt]
    (lB.south) .. controls (-1.90, 0.40) and (-2.30, 0.20) .. (-2.30, 0.00);
  \draw[->, twcoral, line width=1pt]
    (lB.south) .. controls (-0.90, 0.40) and (-0.50, 0.20) .. (-0.50, 0.00);
  \node[annotC, anchor=west] at (0.90, -0.15) {varying\\partition};

  \draw[tedge] (-2.30, -0.45) -- (llB.north);
  \draw[tedge] (-0.50, -0.45) -- (lrB.north);

  \node[bottomcap] at (-0.2, -1.70) {$x_1$ evaluated $1\times$,\quad $x_5$ partitioned $1\times$};
\end{scope}
\end{tikzpicture}
\caption{Worked example: a 4-row group with shared $x_1{=}2$
and varying $x_5{=}[2,4,7,9]$.
(a)~Standard inference evaluates $x_1 \le 3$ four times and
$x_5 \le 5$ four times.
(b)~TreeWalker evaluates the constant split once, carries the
active-row mask $[1111]$, and partitions it at the varying split
into $[1100]$ and $[0011]$ before leaf accumulation.}\label{fig:worked_example}
\end{figure}

\paragraph{Constant walk.}
Starting from the root, TreeWalker follows the path determined by the
constant features. At each node, if the split feature is constant
across the group, the split outcome is identical for all rows and the
algorithm advances to the appropriate child in $O(1)$. The walk
continues until it reaches either a leaf or a varying split.

\paragraph{Varying split.}
When the walk encounters a node whose split feature varies across the
group, TreeWalker partitions the active row bitmask: for each active
row it evaluates the split condition and assigns the row to the left or
right child's bitmask. The walk then recurses on both children, each
carrying its subset of active rows.

\paragraph{Leaf accumulation.}
At a leaf, TreeWalker adds the leaf value to the output for every
active row in the bitmask.

\begin{algorithm}[!htb]
\caption{Recursive partial evaluation of one tree for a group of $G$
rows. Analyzed in \S\ref{sec:analysis}; the engineered variant
(Algorithm~\ref{alg:partial_eval_impl},
Appendix~\ref{sec:app_implementation}) applies the constant-factor
optimizations of \S\ref{sec:tuning}.}\label{alg:partial_eval}
\begin{algorithmic}[1]
\Procedure{PartialEval}{tree, node, mask} \Comment{\texttt{mask} is a \texttt{u16}/\texttt{u32}/\texttt{u64} or word array}
\While{node is not a leaf}
    \If{node.feature is constant for the group}
        \State node $\gets$ child matching constant value \Comment{constant walk}
    \Else
        \State $(m_L, m_R) \gets \Call{Partition}{\text{node, mask}}$ \Comment{varying split}
        \State \Call{PartialEval}{tree, node.left,  $m_L$}
        \State \Call{PartialEval}{tree, node.right, $m_R$}
        \State \Return
    \EndIf
\EndWhile
\ForAll{bit $r$ set in mask} \Comment{leaf accumulation}
    \State $\text{results}[r] \gets \text{results}[r] + \text{node.value}$
\EndFor
\EndProcedure
\end{algorithmic}
\end{algorithm}

\subsection{Implementation Tuning}\label{sec:tuning}

The conceptual algorithm admits four constant-factor optimizations
that the implementation applies and the analysis
(\S\ref{sec:analysis}) references: an \emph{unsplit shortcut} skips
recursion into empty-mask children (Lemma~\ref{lem:unsplit});
\emph{heavy-path DFS} reorders nodes for cache locality, with
light-side recursion bounding stack depth (Lemma~\ref{lem:stack});
\emph{monotonic early-termination} stops partitioning when the split
feature is monotone along the row sequence (trace evaluator only); and \emph{predicate
precomputation} (sort + sweep per varying feature, bitmask lookup at
each split) instantiates Corollary~\ref{cor:impl}. The engineered
variant (Algorithm~\ref{alg:partial_eval_impl}) and data layout
(16-byte node, contiguous arrays) appear in
Appendix~\ref{sec:app_implementation}.

\section{Analysis}\label{sec:analysis}

The analysis has three steps. Theorem~\ref{thm:decomp} writes
TreeWalker's per-tree work as an identity on every group: structural
terms (visited splits, leaf writes) plus a single \emph{predicate-mask
provisioning cost} $Q$ for supplying varying-feature partition masks.
Corollary~\ref{cor:trace} instantiates $Q$ for the trace evaluator and
yields the asymptote $1/(1-\bar\rho_G)$. Corollary~\ref{cor:impl}
instantiates $Q$ for the default precompute evaluator and bounds
forest-level work. \S\ref{sec:decomp_validation} checks both against
emitted counters.

\subsection{Setup and Definitions}\label{sec:theorem}\label{sec:complexity}

Fix a single binary tree~$\tau$ of depth at most~$L$. A \emph{group} is
a set $X = \{x_1, \ldots, x_G\}$ of rows sharing a constant-feature set
$F_c \subseteq \{1, \ldots, n\}$; the remaining features $F_v$ vary
across rows. For each node $u$ of $\tau$, write $A(u) \subseteq [G]$
for the \emph{active row set}: the rows that reach $u$ under
TreeWalker's partial evaluation of $\tau$ on $X$.

\begin{definition}[Constant projection]
$T_c(\tau, X)$ is the directed subgraph of $\tau$ obtained by deleting
the child not selected by the group's constant value at each
constant-feature split, while retaining both children at each
varying-feature split. By construction, every node visited under partial
evaluation lies in $T_c(\tau, X)$.
\end{definition}

\begin{definition}[Trace counters]
\begin{align*}
  C(\tau, X) &= \bigl|\{u \in \tau : u \text{ constant split}, A(u) \ne \emptyset\}\bigr| \\
  V(\tau, X) &= \bigl|\{u \in \tau : u \text{ varying split}, A(u) \ne \emptyset\}\bigr|
\end{align*}
\end{definition}

\begin{definition}[Path statistics]
For a row $x \in X$, let $d_\tau(x)$ be its root-to-leaf depth in $\tau$
(number of split nodes on its path), and $d_{v,\tau}(x)$ the number of
varying splits on that path. Define group averages
	$d(\tau, X) = \frac{1}{G} \sum_x d_\tau(x)$ and
	$d_v(\tau, X) = \frac{1}{G} \sum_x d_{v,\tau}(x)$.
\end{definition}

\paragraph{Cost model.}
We work in an \emph{abstract unit-cost model}: $W(\tau, X)$ counts
split-node visits, leaf writes, and predicate-mask provisioning
operations, each weighted as one unit. Wall-clock time is bounded above
by $W$ up to machine-dependent constants for each component (one
comparison per split visit, one $G$-bit mask op per leaf write, etc.);
we suppress these constants in the abstract bound. Charging one unit
per $G$-bit mask operation holds for groups of up to $64$ rows, whose
mask is one 16- to 64-bit word. Wider groups use masks of $K$ 64-bit
words, which make each mask operation $\Theta(K)$. The engine takes
groups of any width this way, in pieces of at most $1{,}024$ rows
(Appendix~\ref{sec:app_chunked}).

\paragraph{Predicate-mask provisioning.}
When queried at a varying split $u$ with active mask $A(u)$, a
\emph{predicate-mask provisioning step} computes the partition
$(A_L(u), A_R(u))$. Let $Q(\tau, X)$ be the aggregate cost of computing
the queried predicate masks; Corollaries~\ref{cor:trace}--\ref{cor:impl}
instantiate two strategies.

\subsection{Theorem}\label{sec:theorem_main}

\begin{theorem}[Structural work decomposition]\label{thm:decomp}
For any tree $\tau$ and group $X$ of size $G$, in the unit-cost model
above,
\begin{equation}\label{eq:decomp}
  W(\tau, X)
  \;=\;
  C(\tau, X) + V(\tau, X) + G + Q(\tau, X)
  \;\le\;
  |T_c(\tau, X)| + G + Q(\tau, X).
\end{equation}
\end{theorem}

\paragraph{Proof sketch.}
Every visited split node contributes one unit of control work
(Lemma~\ref{lem:disjoint}: the unsplit-enabled evaluator visits each
node with non-empty active mask exactly once), giving $C + V$ in total.
Every row reaches exactly one leaf per tree, contributing $G$ leaf
writes. Provisioning queries at varying splits aggregate to $Q$.
Equality at unit cost follows; the structural inequality $C + V \le
|T_c|$ is Lemma~\ref{lem:visit-bound}. Full proofs of
Lemmas~\ref{lem:disjoint}--\ref{lem:stack} are in
Appendix~\ref{sec:app_proof}.

\subsection{Corollaries}\label{sec:corollaries}

\begin{corollary}[Trace evaluator]\label{cor:trace}
Under \emph{trace provisioning}, where each varying predicate is
evaluated per active row at the visit, $Q_{\mathrm{trace}}(\tau, X) =
\sum_{x \in X} d_{v,\tau}(x) = G \cdot d_v(\tau, X).$ Let the
row-independent baseline be $W_{\mathrm{full}}(\tau, X) = G \cdot
(d(\tau, X) + 1)$. Then
\begin{equation}\label{eq:trace-ratio}
  \frac{W_{\mathrm{trace}}(\tau, X)}{G(d + 1)}
  \;\le\;
  \frac{|T_c(\tau, X)|}{G(d + 1)}
  \;+\;
  \frac{d_v + 1}{d + 1}.
\end{equation}
For group distributions $\{\mathcal{D}_G\}$ with $|T_c|, d, d_v$ uniformly
bounded, define $1 - \bar\rho_G = \mathbb{E}[d_v + 1]/\mathbb{E}[d + 1]$. As
$G \to \infty$ the structural term vanishes and the per-row trace ratio
is asymptotically bounded by $1 - \bar\rho_G$, with asymptotic speedup
$1/(1 - \bar\rho_G) = \mathbb{E}[d + 1]/\mathbb{E}[d_v + 1]$.
\end{corollary}

\begin{corollary}[Default implementation]\label{cor:impl}
Under \emph{precompute provisioning}, the forest-level strategy used in
TreeWalker's default path (sort the $G$ group values per varying
feature, sweep against the deduplicated predicate set $\Pi_f$, emit a
bitmask per predicate), for an ensemble $\mathcal{F} = \{\tau_1,
\ldots, \tau_M\}$,
\begin{equation}\label{eq:Q-pre}
  Q_{\mathrm{pre}}(\mathcal{F}, X)
  \;=\;
  O\!\Biggl(
    \sum_{f \in F_v} \bigl( G \log G + |\Pi_f^{\mathrm{num}}| +
      G \cdot |\Pi_f^{\mathrm{cat}}| \bigr)
  \Biggr),
\end{equation}
paid once per group (not per tree). Total ensemble work is bounded by
$\sum_{i=1}^{M}\bigl(|T_c(\tau_i, X)| + G\bigr) +
Q_{\mathrm{pre}}(\mathcal{F}, X).$
\end{corollary}

\paragraph{Note on cross-tree sharing.}
With prefix grouping disabled and predicate sharing disabled, the
ensemble bound reduces to the sum of per-tree statements of
Theorem~\ref{thm:decomp}. The default forest path strictly improves on
that sum: predicate sharing collapses $|\Pi_f|$ across trees, and
prefix grouping reduces $\sum_i (C(\tau_i, X) + V(\tau_i, X))$ via
shared visits at the prefix.

\paragraph{The $1/L$ floor.}
Under the per-row normalization $W/(G \cdot T \cdot L)$ of
\S\ref{sec:decomp_validation}, the $G$ leaf writes per tree alone fix a
structural floor of $1/L$. The provisioning share does not vanish as
$G \to \infty$: per row, the sort term grows as $\log G$ and the
categorical term stays constant. We claim only that this share is small
at measured $G$, because $Q_{\mathrm{pre}}$ is paid once per group
across all $T$ trees: in the widest measured groups the precompute
ratio is $12\%$ (FLCHAIN, $128$ rows) to $14\%$ (SUPPORT, $104$ rows)
above the $1/L$ floor (Table~\ref{tab:decomp-results}). That ratio uses the row-evaluation
proxy for $Q_{\mathrm{pre}}$, which omits the sort and predicate-sweep
terms (Table~\ref{tab:decomp-counters}), so it understates the true
provisioning share.

\subsection{Empirical Decomposition Validation}\label{sec:decomp_validation}

Table~\ref{tab:decomp-results} reports the per-row work ratio
$W(\mathcal{F}, X) / (G \cdot T \cdot L)$ at the reference
configuration ($T=500$, $L=8$, LightGBM, Intel) on the fixed-$G$
datasets. The sweep covers $G \in \{4, 16, 64, 128\}$ for both the
trace and precompute provisioning strategies; the rightmost column
gives the Corollary~\ref{cor:trace} asymptote. The ratio decreases monotonically in $G$ for both
strategies on both datasets and approaches its predicted limit:
$(d_v+1)/L$ for trace (dataset-dependent, $d_v$ varies) and $1/L$ for
precompute (dataset-independent). At $G=128$ the trace ratio is within $5\%$ of the
predicted limit on both datasets; precompute is within $12\%$ (FLCHAIN) and $14\%$ (SUPPORT) of its $1/L$ floor. The table divides by $G \cdot T \cdot L$ rather than
Corollary~\ref{cor:trace}'s $G(d+1)$, so the trace limit becomes
$(d_v+1)/L$.
Figure~\ref{fig:asymptotic-convergence} (Appendix~\ref{sec:app_sensitivity})
plots the same data on a log-$G$ axis with horizontal asymptote lines.

\begin{table}[h]
\centering
\small
\begin{tabular*}{\textwidth}{@{\extracolsep{\fill}}llrrrrr@{}}
\toprule
& & \multicolumn{4}{c}{Group size $G$} & Asymptote \\
\cmidrule(lr){3-6}
Dataset & Evaluator & 4 & 16 & 64 & 128 & $G \to \infty$ \\
\midrule
  SUPPORT & trace & 0.674 & 0.432 & 0.361 & 0.351 & \textit{0.334} \\
  SUPPORT & precompute (default) & 0.507 & 0.226 & 0.154 & 0.143 & \textit{0.125} \\
  FLCHAIN & trace & 0.544 & 0.451 & 0.439 & 0.431 & \textit{0.417} \\
  FLCHAIN & precompute (default) & 0.420 & 0.213 & 0.153 & 0.140 & \textit{0.125} \\
\bottomrule
\end{tabular*}
\caption{Per-row work ratio $W(\mathcal{F}, X) / (G \cdot T \cdot L)$
at the reference configuration ($T=500$, $L=8$, LightGBM,
Intel) on the fixed-$G$ datasets, swept across $G \in \{4, 16, 64, 128\}$.
The final column gives the Corollary~\ref{cor:trace} asymptote
under this baseline: $(d_v+1)/L$ for trace (with $d_v$ estimated from
\texttt{partition\_row\_evals} at $G=128$) and $1/L$ for precompute
(the leaf-write floor). Both
evaluators decrease monotonically in $G$ and approach their predicted
limits. XGBoost numbers (not shown) are within $6\%$ of LightGBM
except SUPPORT trace at $G{=}64$ ($13.5\%$ higher). SUPPORT's horizons
of $64$ and $128$ hold $58$ and $104$ rows (\S\ref{sec:datasets}). Expedia
is excluded from this sweep because it has variable group sizes; it
appears alongside the speedup grid in \S\ref{sec:evaluation}.}
\label{tab:decomp-results}
\end{table}

\section{Evaluation}\label{sec:evaluation}

\subsection{Experimental Setup}\label{sec:setup}

We evaluate TreeWalker on two survival datasets (SUPPORT, FLCHAIN) and
one CTR dataset (Expedia) across a factorial grid of model parameters,
and on a scenario-analysis workload (\S\ref{sec:scenario}).
The reference configuration ($T{=}500$, $L{=}8$, $G{=}16$) anchors the
grid at the mid-sweep; $G{=}16$ is the median survival horizon. We
compare TreeWalker against five external baselines
(LightGBM~\citep{ke2017lightgbm} and XGBoost native predict, tl2cgen,
lleaves, QuickScorer) plus an
internal full-walk baseline. The full-walk baseline isolates the
partial-evaluation contribution from data-layout effects,
distinguishing \emph{algorithmic speedup} (TreeWalker vs.\ full walk)
from \emph{system speedup} (TreeWalker vs.\ external baselines). Hardware (x86-64 Intel Xeon and Arm Neoverse V2 GCE VMs), software versions, compilation flags, the method
roster (Table~\ref{tab:methods}), and the blocked paired measurement
protocol are in Appendix~\ref{sec:app_details}; numerical precision is
in Appendix~\ref{sec:app_precision}. Latency is the median of
per-block medians in $\mu$s per observation (group); speedups are
ratios of latencies with paired bootstrap $95\%$ intervals, and
p5--p95 over blocks shows dispersion.

\subsection{Datasets}\label{sec:datasets}

Table~\ref{tab:datasets} summarizes the three factorial-grid datasets.

\begin{table}[h]
\centering
\caption{Datasets.  The constant fraction $|F_c|/(|F_c|+|F_v|)$ is a feature-level proxy for shareable work; the realized work-share is captured by the path statistic $(d_v+1)/L$ in Table~\ref{tab:decomp-results}. The scenario-analysis dataset is described in \S\ref{sec:scenario}.}\label{tab:datasets}
\begin{tabularx}{\textwidth}{@{}llrrrrr@{}}
\toprule
Dataset & Type & Groups & Rows/Group & Const Feat. & Vary Feat. & Const\% \\
\midrule
SUPPORT & Survival & 1{,}774 & Swept (1--104)   & 14 & 6  & 70\% \\
FLCHAIN & Survival & 1{,}304 & Swept (1--1{,}008) & 8  & 6  & 57\% \\
Expedia & CTR      & 10{,}000 & Variable (5--37) & 10 & 11 & 48\% \\
\bottomrule
\end{tabularx}
\end{table}

SUPPORT~\citep{knaus1995support} and FLCHAIN~\citep{dispenzieri2012flchain}
are standard survival datasets discretized into panel format. Each
observation expands into $G$ rows over a time horizon; six varying
features encode the time step, the remaining steps, the squared
elapsed and remaining fractions, one covariate--time interaction, and
a smooth hazard-shape term; SUPPORT's discretization yields fewer bins
than requested, so its horizons of $64$ and $128$ hold $58$ and $104$
rows. FLCHAIN also runs horizons of $256$ to $1{,}024$ (up to
$1{,}008$ rows). Expedia~\citep{expedia2013kaggle} contains hotel search
sessions where each session (group) consists of search results (rows). Constant features describe the user and search context;
varying features describe each hotel result (price, star rating,
location score). We subsample $50{,}000$ sessions with 2--128 results (seed $42$); an
$80/20$ train/test split yields $10{,}000$ test groups of $5$--$37$ rows.

\subsection{Algorithmic Speedup}\label{sec:factorial}

We sweep a factorial grid: datasets $D \in \{\text{SUPPORT},
\text{FLCHAIN}, \text{Expedia}\}$, frameworks $F \in \{\text{LightGBM},
\text{XGBoost}\}$, tree counts $T \in \{50, 500, 1000, 2000\}$, max
depths $L \in \{2, 4, 8, 16\}$, and group widths $G \in \{1, 2, 4, 8,
16, 32, 64, 128\}$ (survival only; Expedia uses empirical group sizes),
yielding $2{,}449$ measurement rows per architecture over $544$ cells.
Expedia cohorts, the encoded Expedia variant
(Appendix~\ref{sec:app_quickscorer}) and FLCHAIN's wide groups
(\S\ref{sec:sensitivity}) bring the grid to $712$ cells.

Figure~\ref{fig:speedup_heatmap} shows the speedup of TreeWalker over
the full-walk baseline at $G{=}16$ (survival) and empirical groups
(Expedia), aggregated across frameworks. Speedup exceeds $1.0\times$
in every $(T, L)$ heatmap cell, reaching $4.3\times$ (FLCHAIN, $T{=}500$, $L{=}2$); across the $544$ cells (all $G \le 128$, Intel), the largest algorithmic speedup is $9.1\times$ at $G{=}128$ on FLCHAIN
with XGBoost; Arm speedups are $0.84\times$ Intel's at the median of these cells (\S\ref{sec:sensitivity}). Outside the heatmap, speedup falls below $1\times$ mainly at
$G{=}1$ (no amortization) and at $G{=}2$ where workspace overhead
dominates --- the boundary conditions predicted by
Corollary~\ref{cor:trace}.

\begin{figure}[!htb]
\centering
\begin{tikzpicture}[
  x=7.5mm, y=7mm,
  font=\footnotesize,
  cell/.style={
    rectangle, minimum width=7.5mm, minimum height=7mm,
    draw=white, line width=0.4pt,
    inner sep=0pt, anchor=center,
  },
  rowlabel/.style={anchor=east, font=\scriptsize, inner sep=2pt},
  collabel/.style={anchor=north, font=\scriptsize, inner sep=2pt},
  ptitle/.style={font=\small, anchor=south},
  axlab/.style={font=\scriptsize},
]

\node[ptitle] at (1.5, 3.65) {(a)~FLCHAIN};
\node[cell, fill={rgb,255:red,46;green,91;blue,110}, text=white] at (0, 0) {3.77};
\node[cell, fill={rgb,255:red,40;green,39;blue,42}, text=white] at (1, 0) {4.28};
\node[cell, fill={rgb,255:red,42;green,59;blue,68}, text=white] at (2, 0) {4.09};
\node[cell, fill={rgb,255:red,43;green,61;blue,71}, text=white] at (3, 0) {4.06};
\node[cell, fill={rgb,255:red,49;green,102;blue,106}, text=white] at (0, 1) {3.66};
\node[cell, fill={rgb,255:red,50;green,103;blue,106}, text=white] at (1, 1) {3.64};
\node[cell, fill={rgb,255:red,50;green,104;blue,106}, text=white] at (2, 1) {3.63};
\node[cell, fill={rgb,255:red,48;green,98;blue,108}, text=white] at (3, 1) {3.70};
\node[cell, fill={rgb,255:red,124;green,164;blue,79}, text=black] at (0, 2) {2.95};
\node[cell, fill={rgb,255:red,126;green,165;blue,78}, text=black] at (1, 2) {2.94};
\node[cell, fill={rgb,255:red,128;green,166;blue,78}, text=black] at (2, 2) {2.93};
\node[cell, fill={rgb,255:red,142;green,172;blue,74}, text=black] at (3, 2) {2.86};
\node[cell, fill={rgb,255:red,228;green,184;blue,120}, text=black] at (0, 3) {2.30};
\node[cell, fill={rgb,255:red,215;green,183;blue,109}, text=black] at (1, 3) {2.39};
\node[cell, fill={rgb,255:red,214;green,183;blue,109}, text=black] at (2, 3) {2.40};
\node[cell, fill={rgb,255:red,209;green,183;blue,104}, text=black] at (3, 3) {2.43};
\node[collabel] at (0, -0.5) {50};
\node[collabel] at (1, -0.5) {500};
\node[collabel] at (2, -0.5) {1000};
\node[collabel] at (3, -0.5) {2000};
\node[rowlabel] at (-0.5, 0) {2};
\node[rowlabel] at (-0.5, 1) {4};
\node[rowlabel] at (-0.5, 2) {8};
\node[rowlabel] at (-0.5, 3) {16};
\node[rotate=90, axlab] at (-1.35, 1.5) {Max depth ($L$)};

\begin{scope}[xshift=3.7cm]
  \node[ptitle] at (1.5, 3.65) {(b)~SUPPORT};
  \node[cell, fill={rgb,255:red,78;green,146;blue,90}, text=black] at (0, 0) {3.18};
  \node[cell, fill={rgb,255:red,41;green,49;blue,55}, text=white] at (1, 0) {4.18};
  \node[cell, fill={rgb,255:red,43;green,63;blue,73}, text=white] at (2, 0) {4.05};
  \node[cell, fill={rgb,255:red,42;green,59;blue,68}, text=white] at (3, 0) {4.09};
  \node[cell, fill={rgb,255:red,211;green,183;blue,106}, text=black] at (0, 1) {2.42};
  \node[cell, fill={rgb,255:red,48;green,97;blue,108}, text=white] at (1, 1) {3.71};
  \node[cell, fill={rgb,255:red,45;green,78;blue,93}, text=white] at (2, 1) {3.90};
  \node[cell, fill={rgb,255:red,45;green,81;blue,97}, text=white] at (3, 1) {3.87};
  \node[cell, fill={rgb,255:red,237;green,187;blue,135}, text=black] at (0, 2) {2.18};
  \node[cell, fill={rgb,255:red,126;green,165;blue,78}, text=black] at (1, 2) {2.94};
  \node[cell, fill={rgb,255:red,155;green,177;blue,71}, text=black] at (2, 2) {2.80};
  \node[cell, fill={rgb,255:red,183;green,182;blue,82}, text=black] at (3, 2) {2.63};
  \node[cell, fill={rgb,255:red,233;green,208;blue,190}, text=black] at (0, 3) {1.79};
  \node[cell, fill={rgb,255:red,234;green,199;blue,167}, text=black] at (1, 3) {1.95};
  \node[cell, fill={rgb,255:red,234;green,201;blue,173}, text=black] at (2, 3) {1.91};
  \node[cell, fill={rgb,255:red,235;green,195;blue,157}, text=black] at (3, 3) {2.03};
  \node[collabel] at (0, -0.5) {50};
  \node[collabel] at (1, -0.5) {500};
  \node[collabel] at (2, -0.5) {1000};
  \node[collabel] at (3, -0.5) {2000};
\end{scope}

\begin{scope}[xshift=7.4cm]
  \node[ptitle] at (1.5, 3.65) {(c)~Expedia};
  \node[cell, fill={rgb,255:red,226;green,184;blue,118}, text=black] at (0, 0) {2.31};
  \node[cell, fill={rgb,255:red,95;green,153;blue,86}, text=black] at (1, 0) {3.09};
  \node[cell, fill={rgb,255:red,56;green,122;blue,100}, text=white] at (2, 0) {3.43};
  \node[cell, fill={rgb,255:red,47;green,94;blue,109}, text=white] at (3, 0) {3.74};
  \node[cell, fill={rgb,255:red,236;green,188;blue,137}, text=black] at (0, 1) {2.16};
  \node[cell, fill={rgb,255:red,57;green,128;blue,98}, text=white] at (1, 1) {3.38};
  \node[cell, fill={rgb,255:red,48;green,96;blue,108}, text=white] at (2, 1) {3.72};
  \node[cell, fill={rgb,255:red,43;green,66;blue,77}, text=white] at (3, 1) {4.02};
  \node[cell, fill={rgb,255:red,233;green,217;blue,207}, text=black] at (0, 2) {1.62};
  \node[cell, fill={rgb,255:red,147;green,174;blue,73}, text=black] at (1, 2) {2.84};
  \node[cell, fill={rgb,255:red,114;green,160;blue,81}, text=black] at (2, 2) {3.00};
  \node[cell, fill={rgb,255:red,104;green,156;blue,83}, text=black] at (3, 2) {3.05};
  \node[cell, fill={rgb,255:red,237;green,237;blue,237}, text=black] at (0, 3) {1.20};
  \node[cell, fill={rgb,255:red,232;green,210;blue,196}, text=black] at (1, 3) {1.75};
  \node[cell, fill={rgb,255:red,234;green,203;blue,176}, text=black] at (2, 3) {1.89};
  \node[cell, fill={rgb,255:red,235;green,198;blue,163}, text=black] at (3, 3) {1.98};
  \node[collabel] at (0, -0.5) {50};
  \node[collabel] at (1, -0.5) {500};
  \node[collabel] at (2, -0.5) {1000};
  \node[collabel] at (3, -0.5) {2000};
\end{scope}

\node[axlab] at (6.43, -1.4) {Number of trees ($T$)};

\begin{scope}[xshift=11.1cm, yshift=-7mm, x=1cm, y=1cm]
  \node[font=\scriptsize, anchor=south west] at (0, 3.65) {Speedup};
  \fill[fill={rgb,255:red,237;green,237;blue,237}] (0, 0.0000) rectangle (5mm, 0.0292);
  \fill[fill={rgb,255:red,237;green,236;blue,235}] (0, 0.0292) rectangle (5mm, 0.0583);
  \fill[fill={rgb,255:red,236;green,234;blue,233}] (0, 0.0583) rectangle (5mm, 0.0875);
  \fill[fill={rgb,255:red,236;green,233;blue,231}] (0, 0.0875) rectangle (5mm, 0.1167);
  \fill[fill={rgb,255:red,236;green,232;blue,230}] (0, 0.1167) rectangle (5mm, 0.1458);
  \fill[fill={rgb,255:red,236;green,231;blue,228}] (0, 0.1458) rectangle (5mm, 0.1750);
  \fill[fill={rgb,255:red,235;green,229;blue,226}] (0, 0.1750) rectangle (5mm, 0.2042);
  \fill[fill={rgb,255:red,235;green,228;blue,224}] (0, 0.2042) rectangle (5mm, 0.2333);
  \fill[fill={rgb,255:red,235;green,227;blue,222}] (0, 0.2333) rectangle (5mm, 0.2625);
  \fill[fill={rgb,255:red,235;green,226;blue,220}] (0, 0.2625) rectangle (5mm, 0.2917);
  \fill[fill={rgb,255:red,234;green,224;blue,218}] (0, 0.2917) rectangle (5mm, 0.3208);
  \fill[fill={rgb,255:red,234;green,223;blue,216}] (0, 0.3208) rectangle (5mm, 0.3500);
  \fill[fill={rgb,255:red,234;green,222;blue,215}] (0, 0.3500) rectangle (5mm, 0.3792);
  \fill[fill={rgb,255:red,234;green,221;blue,213}] (0, 0.3792) rectangle (5mm, 0.4083);
  \fill[fill={rgb,255:red,233;green,219;blue,211}] (0, 0.4083) rectangle (5mm, 0.4375);
  \fill[fill={rgb,255:red,233;green,218;blue,209}] (0, 0.4375) rectangle (5mm, 0.4667);
  \fill[fill={rgb,255:red,233;green,217;blue,207}] (0, 0.4667) rectangle (5mm, 0.4958);
  \fill[fill={rgb,255:red,233;green,216;blue,205}] (0, 0.4958) rectangle (5mm, 0.5250);
  \fill[fill={rgb,255:red,232;green,214;blue,203}] (0, 0.5250) rectangle (5mm, 0.5542);
  \fill[fill={rgb,255:red,232;green,213;blue,202}] (0, 0.5542) rectangle (5mm, 0.5833);
  \fill[fill={rgb,255:red,232;green,212;blue,199}] (0, 0.5833) rectangle (5mm, 0.6125);
  \fill[fill={rgb,255:red,232;green,210;blue,196}] (0, 0.6125) rectangle (5mm, 0.6417);
  \fill[fill={rgb,255:red,233;green,209;blue,192}] (0, 0.6417) rectangle (5mm, 0.6708);
  \fill[fill={rgb,255:red,233;green,208;blue,189}] (0, 0.6708) rectangle (5mm, 0.7000);
  \fill[fill={rgb,255:red,233;green,206;blue,185}] (0, 0.7000) rectangle (5mm, 0.7292);
  \fill[fill={rgb,255:red,233;green,205;blue,181}] (0, 0.7292) rectangle (5mm, 0.7583);
  \fill[fill={rgb,255:red,234;green,203;blue,178}] (0, 0.7583) rectangle (5mm, 0.7875);
  \fill[fill={rgb,255:red,234;green,202;blue,174}] (0, 0.7875) rectangle (5mm, 0.8167);
  \fill[fill={rgb,255:red,234;green,200;blue,170}] (0, 0.8167) rectangle (5mm, 0.8458);
  \fill[fill={rgb,255:red,234;green,199;blue,167}] (0, 0.8458) rectangle (5mm, 0.8750);
  \fill[fill={rgb,255:red,235;green,198;blue,163}] (0, 0.8750) rectangle (5mm, 0.9042);
  \fill[fill={rgb,255:red,235;green,196;blue,159}] (0, 0.9042) rectangle (5mm, 0.9333);
  \fill[fill={rgb,255:red,235;green,195;blue,156}] (0, 0.9333) rectangle (5mm, 0.9625);
  \fill[fill={rgb,255:red,235;green,193;blue,152}] (0, 0.9625) rectangle (5mm, 0.9917);
  \fill[fill={rgb,255:red,236;green,192;blue,149}] (0, 0.9917) rectangle (5mm, 1.0208);
  \fill[fill={rgb,255:red,236;green,191;blue,145}] (0, 1.0208) rectangle (5mm, 1.0500);
  \fill[fill={rgb,255:red,236;green,189;blue,141}] (0, 1.0500) rectangle (5mm, 1.0792);
  \fill[fill={rgb,255:red,236;green,188;blue,138}] (0, 1.0792) rectangle (5mm, 1.1083);
  \fill[fill={rgb,255:red,237;green,186;blue,134}] (0, 1.1083) rectangle (5mm, 1.1375);
  \fill[fill={rgb,255:red,237;green,185;blue,130}] (0, 1.1375) rectangle (5mm, 1.1667);
  \fill[fill={rgb,255:red,236;green,184;blue,127}] (0, 1.1667) rectangle (5mm, 1.1958);
  \fill[fill={rgb,255:red,232;green,184;blue,124}] (0, 1.1958) rectangle (5mm, 1.2250);
  \fill[fill={rgb,255:red,229;green,184;blue,121}] (0, 1.2250) rectangle (5mm, 1.2542);
  \fill[fill={rgb,255:red,225;green,184;blue,118}] (0, 1.2542) rectangle (5mm, 1.2833);
  \fill[fill={rgb,255:red,222;green,184;blue,115}] (0, 1.2833) rectangle (5mm, 1.3125);
  \fill[fill={rgb,255:red,218;green,183;blue,112}] (0, 1.3125) rectangle (5mm, 1.3417);
  \fill[fill={rgb,255:red,215;green,183;blue,109}] (0, 1.3417) rectangle (5mm, 1.3708);
  \fill[fill={rgb,255:red,211;green,183;blue,106}] (0, 1.3708) rectangle (5mm, 1.4000);
  \fill[fill={rgb,255:red,208;green,183;blue,103}] (0, 1.4000) rectangle (5mm, 1.4292);
  \fill[fill={rgb,255:red,204;green,183;blue,100}] (0, 1.4292) rectangle (5mm, 1.4583);
  \fill[fill={rgb,255:red,201;green,183;blue,97}] (0, 1.4583) rectangle (5mm, 1.4875);
  \fill[fill={rgb,255:red,197;green,183;blue,94}] (0, 1.4875) rectangle (5mm, 1.5167);
  \fill[fill={rgb,255:red,193;green,183;blue,91}] (0, 1.5167) rectangle (5mm, 1.5458);
  \fill[fill={rgb,255:red,190;green,183;blue,88}] (0, 1.5458) rectangle (5mm, 1.5750);
  \fill[fill={rgb,255:red,186;green,183;blue,85}] (0, 1.5750) rectangle (5mm, 1.6042);
  \fill[fill={rgb,255:red,183;green,182;blue,82}] (0, 1.6042) rectangle (5mm, 1.6333);
  \fill[fill={rgb,255:red,179;green,182;blue,79}] (0, 1.6333) rectangle (5mm, 1.6625);
  \fill[fill={rgb,255:red,176;green,182;blue,76}] (0, 1.6625) rectangle (5mm, 1.6917);
  \fill[fill={rgb,255:red,172;green,182;blue,73}] (0, 1.6917) rectangle (5mm, 1.7208);
  \fill[fill={rgb,255:red,169;green,182;blue,70}] (0, 1.7208) rectangle (5mm, 1.7500);
  \fill[fill={rgb,255:red,164;green,181;blue,69}] (0, 1.7500) rectangle (5mm, 1.7792);
  \fill[fill={rgb,255:red,159;green,179;blue,70}] (0, 1.7792) rectangle (5mm, 1.8083);
  \fill[fill={rgb,255:red,154;green,177;blue,71}] (0, 1.8083) rectangle (5mm, 1.8375);
  \fill[fill={rgb,255:red,148;green,174;blue,73}] (0, 1.8375) rectangle (5mm, 1.8667);
  \fill[fill={rgb,255:red,143;green,172;blue,74}] (0, 1.8667) rectangle (5mm, 1.8958);
  \fill[fill={rgb,255:red,138;green,170;blue,75}] (0, 1.8958) rectangle (5mm, 1.9250);
  \fill[fill={rgb,255:red,132;green,168;blue,77}] (0, 1.9250) rectangle (5mm, 1.9542);
  \fill[fill={rgb,255:red,127;green,166;blue,78}] (0, 1.9542) rectangle (5mm, 1.9833);
  \fill[fill={rgb,255:red,122;green,164;blue,79}] (0, 1.9833) rectangle (5mm, 2.0125);
  \fill[fill={rgb,255:red,116;green,161;blue,80}] (0, 2.0125) rectangle (5mm, 2.0417);
  \fill[fill={rgb,255:red,111;green,159;blue,82}] (0, 2.0417) rectangle (5mm, 2.0708);
  \fill[fill={rgb,255:red,106;green,157;blue,83}] (0, 2.0708) rectangle (5mm, 2.1000);
  \fill[fill={rgb,255:red,100;green,155;blue,84}] (0, 2.1000) rectangle (5mm, 2.1292);
  \fill[fill={rgb,255:red,95;green,153;blue,86}] (0, 2.1292) rectangle (5mm, 2.1583);
  \fill[fill={rgb,255:red,90;green,151;blue,87}] (0, 2.1583) rectangle (5mm, 2.1875);
  \fill[fill={rgb,255:red,84;green,148;blue,88}] (0, 2.1875) rectangle (5mm, 2.2167);
  \fill[fill={rgb,255:red,79;green,146;blue,90}] (0, 2.2167) rectangle (5mm, 2.2458);
  \fill[fill={rgb,255:red,73;green,144;blue,91}] (0, 2.2458) rectangle (5mm, 2.2750);
  \fill[fill={rgb,255:red,68;green,142;blue,92}] (0, 2.2750) rectangle (5mm, 2.3042);
  \fill[fill={rgb,255:red,63;green,140;blue,94}] (0, 2.3042) rectangle (5mm, 2.3333);
  \fill[fill={rgb,255:red,60;green,137;blue,95}] (0, 2.3333) rectangle (5mm, 2.3625);
  \fill[fill={rgb,255:red,60;green,135;blue,95}] (0, 2.3625) rectangle (5mm, 2.3917);
  \fill[fill={rgb,255:red,59;green,133;blue,96}] (0, 2.3917) rectangle (5mm, 2.4208);
  \fill[fill={rgb,255:red,58;green,130;blue,97}] (0, 2.4208) rectangle (5mm, 2.4500);
  \fill[fill={rgb,255:red,57;green,128;blue,98}] (0, 2.4500) rectangle (5mm, 2.4792);
  \fill[fill={rgb,255:red,57;green,125;blue,99}] (0, 2.4792) rectangle (5mm, 2.5083);
  \fill[fill={rgb,255:red,56;green,123;blue,99}] (0, 2.5083) rectangle (5mm, 2.5375);
  \fill[fill={rgb,255:red,55;green,120;blue,100}] (0, 2.5375) rectangle (5mm, 2.5667);
  \fill[fill={rgb,255:red,54;green,118;blue,101}] (0, 2.5667) rectangle (5mm, 2.5958);
  \fill[fill={rgb,255:red,54;green,116;blue,102}] (0, 2.5958) rectangle (5mm, 2.6250);
  \fill[fill={rgb,255:red,53;green,113;blue,103}] (0, 2.6250) rectangle (5mm, 2.6542);
  \fill[fill={rgb,255:red,52;green,111;blue,103}] (0, 2.6542) rectangle (5mm, 2.6833);
  \fill[fill={rgb,255:red,51;green,108;blue,104}] (0, 2.6833) rectangle (5mm, 2.7125);
  \fill[fill={rgb,255:red,51;green,106;blue,105}] (0, 2.7125) rectangle (5mm, 2.7417);
  \fill[fill={rgb,255:red,50;green,104;blue,106}] (0, 2.7417) rectangle (5mm, 2.7708);
  \fill[fill={rgb,255:red,49;green,101;blue,107}] (0, 2.7708) rectangle (5mm, 2.8000);
  \fill[fill={rgb,255:red,48;green,99;blue,107}] (0, 2.8000) rectangle (5mm, 2.8292);
  \fill[fill={rgb,255:red,48;green,96;blue,108}] (0, 2.8292) rectangle (5mm, 2.8583);
  \fill[fill={rgb,255:red,47;green,94;blue,109}] (0, 2.8583) rectangle (5mm, 2.8875);
  \fill[fill={rgb,255:red,46;green,91;blue,110}] (0, 2.8875) rectangle (5mm, 2.9167);
  \fill[fill={rgb,255:red,46;green,89;blue,107}] (0, 2.9167) rectangle (5mm, 2.9458);
  \fill[fill={rgb,255:red,45;green,86;blue,104}] (0, 2.9458) rectangle (5mm, 2.9750);
  \fill[fill={rgb,255:red,45;green,84;blue,100}] (0, 2.9750) rectangle (5mm, 3.0042);
  \fill[fill={rgb,255:red,45;green,81;blue,97}] (0, 3.0042) rectangle (5mm, 3.0333);
  \fill[fill={rgb,255:red,45;green,78;blue,93}] (0, 3.0333) rectangle (5mm, 3.0625);
  \fill[fill={rgb,255:red,44;green,76;blue,90}] (0, 3.0625) rectangle (5mm, 3.0917);
  \fill[fill={rgb,255:red,44;green,73;blue,87}] (0, 3.0917) rectangle (5mm, 3.1208);
  \fill[fill={rgb,255:red,44;green,70;blue,83}] (0, 3.1208) rectangle (5mm, 3.1500);
  \fill[fill={rgb,255:red,43;green,68;blue,80}] (0, 3.1500) rectangle (5mm, 3.1792);
  \fill[fill={rgb,255:red,43;green,65;blue,76}] (0, 3.1792) rectangle (5mm, 3.2083);
  \fill[fill={rgb,255:red,43;green,63;blue,73}] (0, 3.2083) rectangle (5mm, 3.2375);
  \fill[fill={rgb,255:red,42;green,60;blue,69}] (0, 3.2375) rectangle (5mm, 3.2667);
  \fill[fill={rgb,255:red,42;green,57;blue,66}] (0, 3.2667) rectangle (5mm, 3.2958);
  \fill[fill={rgb,255:red,42;green,55;blue,63}] (0, 3.2958) rectangle (5mm, 3.3250);
  \fill[fill={rgb,255:red,42;green,52;blue,59}] (0, 3.3250) rectangle (5mm, 3.3542);
  \fill[fill={rgb,255:red,41;green,49;blue,56}] (0, 3.3542) rectangle (5mm, 3.3833);
  \fill[fill={rgb,255:red,41;green,47;blue,52}] (0, 3.3833) rectangle (5mm, 3.4125);
  \fill[fill={rgb,255:red,41;green,44;blue,49}] (0, 3.4125) rectangle (5mm, 3.4417);
  \fill[fill={rgb,255:red,40;green,42;blue,45}] (0, 3.4417) rectangle (5mm, 3.4708);
  \fill[fill={rgb,255:red,40;green,39;blue,42}] (0, 3.4708) rectangle (5mm, 3.5000);
  \node[anchor=west, font=\scriptsize, inner sep=1pt] at (5.5mm, 0.3361) {1.5};
  \node[anchor=west, font=\scriptsize, inner sep=1pt] at (5.5mm, 0.9046) {2.0};
  \node[anchor=west, font=\scriptsize, inner sep=1pt] at (5.5mm, 1.4730) {2.5};
  \node[anchor=west, font=\scriptsize, inner sep=1pt] at (5.5mm, 2.0415) {3.0};
  \node[anchor=west, font=\scriptsize, inner sep=1pt] at (5.5mm, 2.6100) {3.5};
  \node[anchor=west, font=\scriptsize, inner sep=1pt] at (5.5mm, 3.1785) {4.0};
  \draw[line width=0.3pt] (0, 0) rectangle (5mm, 3.5);
\end{scope}

\end{tikzpicture}
\caption{Algorithmic speedup across all $(T, L)$ combinations on Intel,
for (a)~FLCHAIN ($G{=}16$), (b)~SUPPORT ($G{=}16$), and (c)~Expedia
(empirical groups). Each cell: arithmetic mean of the two
framework-specific speedups. Darker = higher speedup.}\label{fig:speedup_heatmap}
\end{figure}
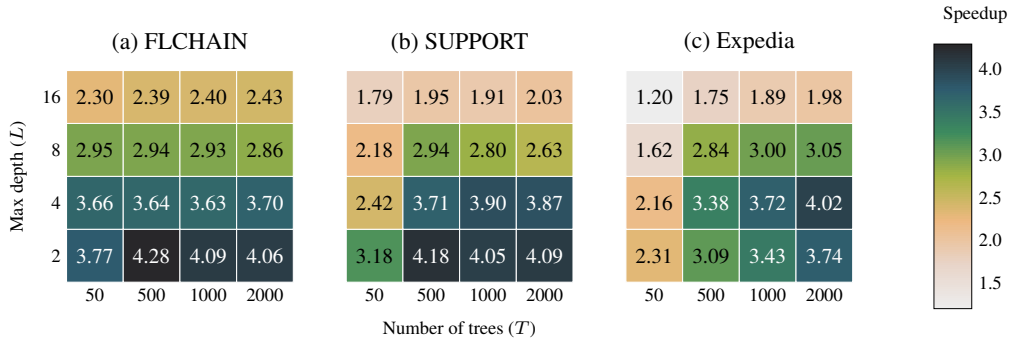

\subsection{System Comparison}\label{sec:system_comparison}

Figure~\ref{fig:method_comparison} compares all methods at the
reference. TreeWalker is $7.1\times$ faster than LightGBM native and
$9.0\times$ faster than XGBoost native on SUPPORT, decomposing into a
$2.5\times$ data-layout effect (full-walk vs.\ LightGBM native) plus a
$2.9\times$ partial-evaluation effect on top. Against tl2cgen and
lleaves (native flags) TreeWalker is $2.1$--$2.3\times$ and $1.1$--$1.2\times$ faster at $G{=}16$ on LightGBM survival; the lleaves
gap grows to $2.1$--$2.6\times$ at $G{=}128$ as group structure
amortizes. Among the $544$ cells, TreeWalker's largest LightGBM-native system speedup is $54\times$ ($T{=}2000, L{=}8, G{=}128$ on SUPPORT, $33$\,ms vs.\ $0.61$\,ms), combining algorithmic speedup with the data-layout/API
gap to native LightGBM. On Expedia/XGBoost at the reference, TreeWalker is
$6.0\times$ faster ($168$ vs.\ $1{,}015$\,$\mu$s). QuickScorer computes a different function
on nearly all of these models (no categorical splits or missing-value routing); the
methods are on orthogonal axes (per-walk vs.\ number-of-walks), and
Appendix~\ref{sec:app_quickscorer} compares them on an encoded Expedia
variant.

\begin{figure}[!htb]
\centering
\includegraphics[width=0.9\textwidth]{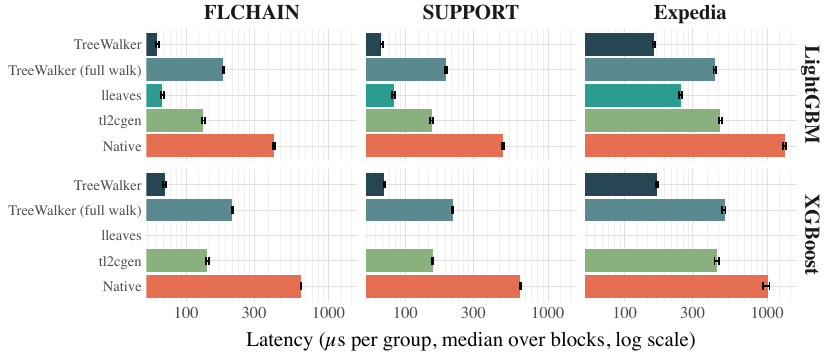}
\caption{Per-observation latency ($\mu$s) for each inference method at the reference configuration ($T{=}500$, $L{=}8$, $G{=}16$), on Intel.  Error bars: p5--p95 over the cell's $48$ measurement blocks.}\label{fig:method_comparison}
\end{figure}

\subsection{Group-Width Sensitivity}\label{sec:sensitivity}

Figure~\ref{fig:horizon_amortization} shows the central sensitivity:
algorithmic speedup grows monotonically with $G$, from $\sim 1\times$
at $G{=}1$ to $7.7$--$9.1\times$ at $G{=}128$ on both survival datasets. This is the
wall-clock counterpart of the work-ratio convergence visualized in
Figure~\ref{fig:asymptotic-convergence} (Appendix~\ref{sec:app_sensitivity}):
the same x-axis ($G$). On the survival datasets, shallower trees give higher speedup
($3.9$--$4.7\times$ at $L{=}2$ vs.\ $1.9$--$2.6\times$ at $L{=}16$);
speedup declines from $T{=}500$ ($2.9$--$3.0\times$) to $T{=}2000$ ($2.5$--$2.9\times$). Arm Neoverse V2 speedups are $0.84\times$ Intel Xeon's at the median cell. Figure~\ref{fig:horizon_amortization}(c) shows Expedia cohorts of $4$--$32$ rows; speedup increases monotonically with cohort size. Detailed plots are in Appendix~\ref{sec:app_sensitivity}.

On FLCHAIN's horizons of $256$ to $1{,}024$ ($1{,}008$ rows), whole groups
reach up to $17.1\times$ algorithmic speedup on Intel ($10.1\times$ on
Arm), and $128$-row chunks take $1.17$--$1.22\times$ the latency of whole groups
on Intel and $1.12$--$1.20\times$ on Arm (medians over models and
frameworks; Appendix~\ref{sec:app_chunked}).

\begin{figure}[!htb]
\centering
\includegraphics[width=0.85\textwidth]{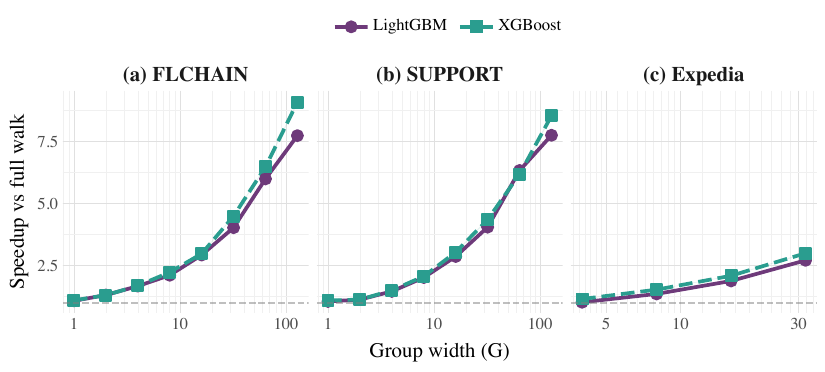}
\caption{Algorithmic speedup vs.\ group width $G$ ($T{=}500$, $L{=}8$, Intel).  (a)~FLCHAIN, (b)~SUPPORT, (c)~Expedia cohorts by size.  Ribbons: $95\%$ intervals.}\label{fig:horizon_amortization}
\end{figure}

\subsection{Scenario Analysis}\label{sec:scenario}

Scenario analysis generalizes the survival workloads: discrete-time
expansion varies the time index, while a what-if query varies an
arbitrary feature subset. We benchmark it on the UCI Default of Credit
Card Clients dataset~\citep{yeh2009credit} ($30{,}000$ rows, $23$
numeric features; LightGBM $T{=}500$, $L{=}8$, seed $42$). Each of
$2{,}000$ test rows expands into $G$ variants that scale $k$ selected
features by log-uniform multipliers in $[0.5, 2]$ and keep the other
$23-k$ features identical. Multiplicative
shocks leave zero-valued entries constant, so $k$ counts selected
features, not realized varying ones: at $k{=}8$ the mean realized
varying dimensionality is $6.80$ of $8$, and we measure $d_v$ on the
realized groups. We sweep $k \in \{1,2,4,8\}$ and
$G \in \{4,16,64,128\}$ on both architectures
(Appendix~\ref{sec:app_scenario}).

\begin{table}[h]
\centering
\small
\caption{Scenario analysis at $G{=}128$: measured trace-evaluator work
ratio $W/(G \cdot T \cdot L)$ against the Corollary~\ref{cor:trace}
prediction under this baseline, $(d_v+1)/L$, and algorithmic speedup over the
same-layout full walk. Work counters are identical on both
architectures. $k$ counts selected perturbed features.}\label{tab:scenario}
\begin{tabular*}{\textwidth}{@{\extracolsep{\fill}}rrrrrr@{}}
\toprule
$k$ & $d_v$ & Measured ratio & Predicted $(d_v+1)/L$ & Speedup (Arm) & Speedup (Intel) \\
\midrule
1 & 0.50 & 0.195 & 0.188 & $9.2\times$ & $15.2\times$ \\
2 & 1.15 & 0.276 & 0.268 & $8.4\times$ & $12.8\times$ \\
4 & 2.09 & 0.395 & 0.386 & $7.4\times$ & $10.3\times$ \\
8 & 3.81 & 0.611 & 0.601 & $6.2\times$ & $8.1\times$ \\
\bottomrule
\end{tabular*}
\end{table}

TreeWalker is faster than the same-layout full walk in all $16$ cells
on each architecture, and its speedup rises with $G$ for every $k$.
Speedup ranges from $1.96$--$15.2\times$ at $k{=}1$ to $1.08$--$8.1\times$
at $k{=}8$ on Intel, and from $1.77$--$9.2\times$ to $1.04$--$6.2\times$
on Arm. At $G{=}128$ the measured work ratio
exceeds the Corollary~\ref{cor:trace} prediction by $1.7$--$3.7\%$
across $k$ (Table~\ref{tab:scenario}).

\subsection{Ablation Study}\label{sec:ablation}

We ablate eight binary flags (unsplit shortcut, monotonic
early-termination, precomputation, tree ordering, prefix grouping,
bitset interning, predicate sweep, predicate deduplication), each
disabled in turn at one anchor per dataset (the reference for SUPPORT
and Expedia, $G{=}4$ for FLCHAIN). The unsplit-shortcut
ablation is the direct empirical analogue of Lemma~\ref{lem:unsplit}:
disabling it lets the algorithm descend into empty-mask subtrees, so
the change in TreeWalker's emitted call count is at least $U(\tau, X) =
V - B$, where $B$ counts varying splits with both children non-empty.
We do not claim a wall-clock multiplier from this lemma; the
call-count effect is structural. Precomputation and predicate deduplication
have the largest wall-clock effects (SUPPORT $+$21--39\% and
$+$28--34\%, FLCHAIN $+$11\% and $+$23--27\%, Expedia $+$93--139\%
and $+$77--82\%); tree ordering adds $+$3--8\% on the survival anchors
and at most $+$1\% on Expedia, and monotonic early-termination,
active only without precomputation, changes sign ($-$3\% to $+$7\%)
(Figure~\ref{fig:ablation_waterfall}, Appendix~\ref{sec:app_sensitivity}).

\section{Discussion}\label{sec:discussion}\label{sec:limitations}\label{sec:generalization}

\paragraph{What the theorem buys.}
Theorem~\ref{thm:decomp}'s exact equality lets the unsplit-shortcut
ablation (\S\ref{sec:ablation}) bound Lemma~\ref{lem:unsplit}'s call-count delta from below, and lets the precompute ablation swap
$Q_{\mathrm{trace}}$ for $Q_{\mathrm{pre}}$ as the corollaries name.

\paragraph{Composability.}
Partial evaluation lives on a third axis (number of walks) and composes
with the per-visit and per-walk axes from \S\ref{sec:gbdt}. Two
compositions remain unimplemented: a QuickScorer-style inner kernel
for the varying subtrees, and a compiled or vectorized inner loop
layered under the per-group partition. The
decomposition predicts that such compositions reduce $|T_c|$, the
per-visit cost, and $Q$ by independent factors.

\paragraph{Relation to compilation-based approaches.}
lleaves~\citep{boehm2021lleaves} and tl2cgen~\citep{tl2cgen2024}
specialize a tree to compiled code at parse time; partial evaluation specializes a tree to
an input \emph{group} at predict time, with constant features as the
binding-time signal; the two compose as the inner loop and the outer
recursion of Algorithm~\ref{alg:partial_eval}. In the Jones, Gomard, and
Sestoft~\citep{jones1993partial} taxonomy this is the same
\emph{specialization-followed-by-residualization} pattern that
Futamura's first projection~\citep{futamura1999partial} described for
staged interpreters; we residualize against group-level constants
rather than program-level ones.

\paragraph{Scope and limitations.}
The implementation is single-threaded CPU and reads scalar models with
an identity or sigmoid output and sum or average aggregation; we
measure binary-logistic GBDTs only. Evaluation covers two survival, one
CTR/ranking, and one scenario-analysis workload. We do not measure the
cost of constructing groups from unstructured batches; the target
workloads supply groups by construction (\S\ref{sec:introduction}). The f32 audit (Appendix~\ref{sec:app_precision}) shows a roundoff improvement
over native f32, not a model-quality property. The per-tree theorem is
the headline statement; cross-tree improvements (prefix grouping,
predicate sharing) only strengthen the bound and are not separately
theorized. The decomposition validates the work identity, not
wall-clock constant factors. The binding-time partition must be known
at predict time; mixed groups (partially overlapping constants) need
finer-grained bookkeeping. The asymptote is reached only at $G
\gtrsim 64$ in our experiments; small-group regimes pay workspace overhead.

\paragraph{Boundary cases.}
TreeWalker is slower than its own full walk in $33/712$ Intel cells
($78/712$ Arm), all in small groups: survival panels of $G \le 2$ (Arm
also $G{=}4$) and Expedia cohorts of $4$--$8$ rows (Arm one of $16$);
$13$ ($28$) are within $\pm 2\%$, and the worst is $1.29\times$ on Intel
and $1.52\times$ on Arm. TreeWalker beats both native libraries in all
$356$ cells of each framework: narrowest $1.01\times$ (Intel) and
$1.06\times$ (Arm) against LightGBM, on Expedia cohorts of $4$, and
$1.18\times$ and $1.50\times$ against XGBoost. Compiled baselines win
mostly at small $G$: tl2cgen is faster in $122/712$ Intel cells
($70/712$ Arm), $92$ of them at $G \le 4$; lleaves in $94/356$
($56/356$), $83$ at $G \le 8$.

\section{Conclusion}\label{sec:conclusion}

We presented TreeWalker, an inference engine that exploits input-side
correlation across the rows of a grouped prediction request through
recursive partial evaluation, paired with a structural work
decomposition (Theorem~\ref{thm:decomp}) that splits per-tree work
into a control term, a leaf-write term, and a predicate-mask
provisioning term. On Intel, across SUPPORT, FLCHAIN, and Expedia, TreeWalker delivers
$2.7$--$9.1\times$ algorithmic speedup over a row-independent baseline
($2.5$--$6.0\times$ on Arm), and it is faster than that baseline in all $16$
scenario-analysis cells on each architecture; the trace-evaluator
asymptote from Corollary~\ref{cor:trace} is within $5\%$ of measurement
at $G{=}128$ on both survival datasets.

Four extensions follow directly: (i)~measured random-forest timings,
since traversal and mask partitioning do not depend on the training
procedure; (ii)~hierarchical subgrouping for groups with partially
overlapping constants; (iii)~composing with QuickScorer could speed up
shallow $L$, where QuickScorer is faster; (iv)~predicate sharing across
requests would amortize $Q_{\mathrm{pre}}$ across a serving session, a
natural fit for survival-curve and ranking pipelines. GPU execution is
out of scope: we target latency-critical CPU serving.

\begin{ack}
This work was funded by Shopify, the employer of both authors. The
authors have no competing interests beyond their employment at Shopify.
\end{ack}

\bibliographystyle{unsrtnat}
\bibliography{references}

\begin{thebibliography}{28}
\providecommand{\natexlab}[1]{#1}
\providecommand{\url}[1]{\texttt{#1}}
\expandafter\ifx\csname urlstyle\endcsname\relax
  \providecommand{\doi}[1]{doi: #1}\else
  \providecommand{\doi}{doi: \begingroup \urlstyle{rm}\Url}\fi

\bibitem[Chen and Guestrin(2016)]{chen2016xgboost}
Tianqi Chen and Carlos Guestrin.
\newblock {XGBoost}: A scalable tree boosting system.
\newblock In \emph{Proceedings of the 22nd ACM SIGKDD International Conference
  on Knowledge Discovery and Data Mining}, pages 785--794, 2016.
\newblock \doi{10.1145/2939672.2939785}.

\bibitem[Boehm(2025)]{boehm2021lleaves}
Simon Boehm.
\newblock lleaves: Compiler for {LightGBM} gradient-boosted trees, based on
  {LLVM} (v1.4.1).
\newblock \url{https://github.com/siboehm/lleaves/tree/v1.4.1}, 2025.

\bibitem[Nakandala et~al.(2020)Nakandala, Saur, Yu, Karanasos, Curino, Weimer,
  and Interlandi]{nakandala2020mlprediction}
Supun Nakandala, Karla Saur, Gyeong-In Yu, Konstantinos Karanasos, Carlo
  Curino, Markus Weimer, and Matteo Interlandi.
\newblock A tensor compiler for unified machine learning prediction serving.
\newblock In \emph{Proceedings of the 14th USENIX Symposium on Operating
  Systems Design and Implementation (OSDI)}, pages 899--917, 2020.
\newblock URL
  \url{https://www.usenix.org/conference/osdi20/presentation/nakandala}.

\bibitem[{NVIDIA}(2026)]{nvidia2024fil}
{NVIDIA}.
\newblock {RAPIDS} {Forest Inference Library} ({FIL}), {cuML} 26.04.
\newblock
  \url{https://github.com/NVIDIA/cuml/tree/v26.04.00/python/cuml/cuml/fil},
  2026.

\bibitem[Asadi et~al.(2014)Asadi, Lin, and de~Vries]{asadi2014runtime}
Nima Asadi, Jimmy Lin, and Arjen~P. de~Vries.
\newblock Runtime optimizations for tree-based machine learning models.
\newblock \emph{IEEE Transactions on Knowledge and Data Engineering},
  26\penalty0 (9):\penalty0 2281--2292, 2014.
\newblock \doi{10.1109/TKDE.2013.73}.

\bibitem[Lucchese et~al.(2015)Lucchese, Nardini, Orlando, Perego, Tonellotto,
  and Venturini]{lucchese2015quickscorer}
Claudio Lucchese, Franco~Maria Nardini, Salvatore Orlando, Raffaele Perego,
  Nicola Tonellotto, and Rossano Venturini.
\newblock {QuickScorer}: A fast algorithm to rank documents with additive
  ensembles of regression trees.
\newblock In \emph{Proceedings of the 38th International ACM SIGIR Conference
  on Research and Development in Information Retrieval}, pages 73--82, 2015.
\newblock \doi{10.1145/2766462.2767733}.

\bibitem[Jones et~al.(1993)Jones, Gomard, and Sestoft]{jones1993partial}
Neil~D. Jones, Carsten~K. Gomard, and Peter Sestoft.
\newblock \emph{Partial Evaluation and Automatic Program Generation}.
\newblock Prentice Hall International Series in Computer Science. Prentice
  Hall, 1993.
\newblock Free PDF: \url{https://www.itu.dk/people/sestoft/pebook/}.

\bibitem[Futamura(1999)]{futamura1999partial}
Yoshihiko Futamura.
\newblock Partial evaluation of computation process---an approach to a
  compiler-compiler.
\newblock \emph{Higher-Order and Symbolic Computation}, 12\penalty0
  (4):\penalty0 381--391, 1999.
\newblock \doi{10.1023/A:1010095604496}.
\newblock Reprint of Futamura 1971; canonical reference for the Futamura
  projections.

\bibitem[W{\"u}rthinger et~al.(2017)W{\"u}rthinger, Wimmer, Humer, W{\"o}{\ss},
  Stadler, Seaton, Duboscq, Simon, and Grimmer]{wuerthinger2017partial}
Thomas W{\"u}rthinger, Christian Wimmer, Christian Humer, Andreas W{\"o}{\ss},
  Lukas Stadler, Chris Seaton, Gilles Duboscq, Doug Simon, and Matthias
  Grimmer.
\newblock Practical partial evaluation for high-performance dynamic language
  runtimes.
\newblock In \emph{Proceedings of the 38th ACM SIGPLAN Conference on
  Programming Language Design and Implementation (PLDI)}, pages 662--676. ACM,
  2017.
\newblock \doi{10.1145/3062341.3062381}.

\bibitem[Singer and Willett(1993)]{singer1993time}
Judith~D. Singer and John~B. Willett.
\newblock It's about time: Using discrete-time survival analysis to study
  duration and the timing of events.
\newblock \emph{Journal of Educational Statistics}, 18\penalty0 (2):\penalty0
  155--195, 1993.
\newblock \doi{10.2307/1165085}.

\bibitem[Lee et~al.(2018)Lee, Zame, Yoon, and van~der Schaar]{lee2018deephit}
Changhee Lee, William~R. Zame, Jinsung Yoon, and Mihaela van~der Schaar.
\newblock {DeepHit}: A deep learning approach to survival analysis with
  competing risks.
\newblock In \emph{Proceedings of the 32nd AAAI Conference on Artificial
  Intelligence}, pages 2314--2321, 2018.
\newblock \doi{10.1609/aaai.v32i1.11842}.

\bibitem[{Expedia}(2013)]{expedia2013kaggle}
{Expedia}.
\newblock Personalize expedia hotel searches --- {ICDM} 2013.
\newblock \url{https://www.kaggle.com/c/expedia-personalized-sort}, 2013.
\newblock Kaggle / ICDM 2013 learning-to-rank competition dataset.

\bibitem[Gelman and Hill(2006)]{gelman2006multilevel}
Andrew Gelman and Jennifer Hill.
\newblock \emph{Data Analysis Using Regression and Multilevel/Hierarchical
  Models}.
\newblock Cambridge University Press, 2006.
\newblock \doi{10.1017/CBO9780511790942}.

\bibitem[Sela and Simonoff(2012)]{sela2012reem}
Rebecca~J. Sela and Jeffrey~S. Simonoff.
\newblock {RE-EM} trees: A data mining approach for longitudinal and clustered
  data.
\newblock \emph{Machine Learning}, 86\penalty0 (2):\penalty0 169--207, 2012.
\newblock \doi{10.1007/s10994-011-5258-3}.

\bibitem[Hajjem et~al.(2014)Hajjem, Bellavance, and Larocque]{hajjem2014merf}
Ahlem Hajjem, Fran{\c{c}}ois Bellavance, and Denis Larocque.
\newblock Mixed-effects random forest for clustered data.
\newblock \emph{Journal of Statistical Computation and Simulation}, 84\penalty0
  (6):\penalty0 1313--1328, 2014.
\newblock \doi{10.1080/00949655.2012.741599}.

\bibitem[Sigrist(2022)]{sigrist2022gpboost}
Fabio Sigrist.
\newblock Gaussian process boosting.
\newblock \emph{Journal of Machine Learning Research}, 23\penalty0
  (232):\penalty0 1--46, 2022.

\bibitem[Makridakis et~al.(2022)Makridakis, Spiliotis, and
  Assimakopoulos]{makridakis2022m5}
Spyros Makridakis, Evangelos Spiliotis, and Vassilios Assimakopoulos.
\newblock {M5} accuracy competition: Results, findings, and conclusions.
\newblock \emph{International Journal of Forecasting}, 38\penalty0
  (4):\penalty0 1346--1364, 2022.
\newblock \doi{10.1016/j.ijforecast.2021.11.013}.

\bibitem[Cho(2024)]{tl2cgen2024}
Hyunsu Cho.
\newblock {TL2cgen}: Model compiler for decision tree ensembles (v1.0.0).
\newblock \url{https://github.com/dmlc/tl2cgen/releases/tag/1.0.0}, 2024.

\bibitem[Cho and Li(2018)]{cho2018treelite}
Hyunsu Cho and Mu~Li.
\newblock Treelite: Toolbox for decision tree deployment.
\newblock In \emph{Proceedings of Machine Learning and Systems (MLSys)}, 2018.
\newblock URL \url{https://mlsys.org/Conferences/doc/2018/196.pdf}.
\newblock Software used at v4.7.0:
  \url{https://github.com/dmlc/treelite/releases/tag/4.7.0}.

\bibitem[Guillame-Bert et~al.(2023)Guillame-Bert, Bruch, Stotz, and
  Pfeifer]{guillame2022tfdf}
Mathieu Guillame-Bert, Sebastian Bruch, Richard Stotz, and Jan Pfeifer.
\newblock Yggdrasil decision forests: A fast and extensible decision forests
  library.
\newblock In \emph{Proceedings of the 29th ACM SIGKDD Conference on Knowledge
  Discovery and Data Mining (KDD)}, pages 4068--4077, 2023.
\newblock \doi{10.1145/3580305.3599933}.

\bibitem[{ONNX}(2026)]{onnxruntime_tree}
{ONNX}.
\newblock {ONNX}: \texttt{ai.onnx.ml.TreeEnsemble} operator (opset v1.21.0).
\newblock
  \url{https://github.com/onnx/onnx/blob/v1.21.0/docs/Operators-ml.md\#aionnxmltreeensemble},
  2026.
\newblock Unifies the deprecated \texttt{TreeEnsembleRegressor} and
  \texttt{TreeEnsembleClassifier} operators.

\bibitem[Lucchese et~al.(2016)Lucchese, Nardini, Orlando, Perego, Tonellotto,
  and Venturini]{lucchese2017vquickscorer}
Claudio Lucchese, Franco~Maria Nardini, Salvatore Orlando, Raffaele Perego,
  Nicola Tonellotto, and Rossano Venturini.
\newblock Exploiting {CPU} {SIMD} extensions to speed-up document scoring with
  tree ensembles.
\newblock In \emph{Proceedings of the 39th International ACM SIGIR Conference
  on Research and Development in Information Retrieval}, pages 833--836, 2016.
\newblock \doi{10.1145/2911451.2914758}.

\bibitem[Ye et~al.(2018)Ye, Zhou, Zou, Gao, and Zhang]{ye2018rapidscorer}
Ting Ye, Hucheng Zhou, Will~Y. Zou, Bin Gao, and Ruofei Zhang.
\newblock {RapidScorer}: Fast tree ensemble evaluation by maximizing
  compactness in data level parallelization.
\newblock In \emph{Proceedings of the 24th ACM SIGKDD International Conference
  on Knowledge Discovery and Data Mining}, pages 941--950, 2018.
\newblock \doi{10.1145/3219819.3219857}.

\bibitem[Manneschi(2025)]{quickscorer_rs}
Leonardo Manneschi.
\newblock quickscorer-rs: {Rust} implementation of {QuickScorer}, 2025.
\newblock URL \url{https://github.com/leomanne/quickscorer-rs}.
\newblock Commit 8314d03; accessed 2026-03-26.

\bibitem[Ke et~al.(2017)Ke, Meng, Finley, Wang, Chen, Ma, Ye, and
  Liu]{ke2017lightgbm}
Guolin Ke, Qi~Meng, Thomas Finley, Taifeng Wang, Wei Chen, Weidong Ma, Qiwei
  Ye, and Tie-Yan Liu.
\newblock {LightGBM}: A highly efficient gradient boosting decision tree.
\newblock In \emph{Advances in Neural Information Processing Systems},
  volume~30, 2017.
\newblock URL
  \url{https://proceedings.neurips.cc/paper/2017/hash/6449f44a102fde848669bdd9eb6b76fa-Abstract.html}.

\bibitem[Knaus et~al.(1995)Knaus, Harrell, Lynn, Goldman, Phillips, Connors,
  Dawson, Fulkerson, Califf, Desbiens, et~al.]{knaus1995support}
William~A. Knaus, Frank~E. Harrell, Joanne Lynn, Lee Goldman, Russell~S.
  Phillips, Alfred~F. Connors, Neal~V. Dawson, William~J. Fulkerson, Robert~M.
  Califf, Norman Desbiens, et~al.
\newblock The {SUPPORT} prognostic model: Objective estimates of survival for
  seriously ill hospitalized adults.
\newblock \emph{Annals of Internal Medicine}, 122\penalty0 (3):\penalty0
  191--203, 1995.
\newblock \doi{10.7326/0003-4819-122-3-199502010-00007}.

\bibitem[Dispenzieri et~al.(2012)Dispenzieri, Katzmann, Kyle, Larson, Therneau,
  Colby, Clark, Mead, Kumar, Melton~III, and Rajkumar]{dispenzieri2012flchain}
Angela Dispenzieri, Jerry~A. Katzmann, Robert~A. Kyle, Dirk~R. Larson, Terry~M.
  Therneau, Colin~L. Colby, Raynell~J. Clark, Graham~P. Mead, Shaji Kumar,
  L.~Joseph Melton~III, and S.~Vincent Rajkumar.
\newblock Use of nonclonal serum immunoglobulin free light chains to predict
  overall survival in the general population.
\newblock \emph{Mayo Clinic Proceedings}, 87\penalty0 (6):\penalty0 517--523,
  2012.
\newblock \doi{10.1016/j.mayocp.2012.03.009}.

\bibitem[Yeh(2009)]{yeh2009credit}
I-Cheng Yeh.
\newblock Default of credit card clients.
\newblock UCI Machine Learning Repository, 2009.
\newblock CC BY 4.0.

\end{thebibliography}

\newpage
\appendix

\section{Experimental Details}\label{sec:app_details}

\subsection{Hardware}\label{sec:app_hardware}

All experiments run on Google Compute Engine VMs in \texttt{us-east4-a},
with SMT disabled (\texttt{threads\_per\_core=1}), the benchmark pinned
to core $0$, and both CPUs at their all-core turbo frequency. Intel's
is set by \texttt{turbo\_mode=ALL\_CORE\_MAX} ($3.1$\,GHz); without it
GCE runs C4 at unrestricted frequency. GCE always runs its Arm
processors at their all-core turbo frequency and offers no setting on
C4A.

\begin{table}[h]
\centering
\caption{Hardware specifications.  Per-core cache sizes; L3 shared.  Both VMs in the same GCE zone, SMT disabled, at the all-core turbo frequency.}\label{tab:hardware_full}
\begin{tabular*}{\textwidth}{@{\extracolsep{\fill}}lll@{}}
\toprule
                    & Intel                            & Arm \\
\midrule
GCE machine type    & c4-standard-32                   & c4a-highmem-16 \\
CPU                 & Xeon Platinum 8581C              & Google Axion (Neoverse V2) \\
Frequency           & 2.3 GHz base, 3.1 GHz all-core turbo & 3.0 GHz \\
Cores               & 16 (SMT disabled)                & 16 \\
L1d / L1i (per core)& 48 KiB / 32 KiB                 & 64 KiB / 64 KiB \\
L2 (per core)       & 2 MiB                            & 2 MiB \\
L3 (shared)         & 260 MiB                          & 80 MiB \\
RAM                 & 120 GB                           & 128 GB \\
OS                  & Ubuntu 26.04.1 LTS               & Ubuntu 26.04.1 LTS \\
Kernel              & 7.0.0-1011-gcp                   & 7.0.0-1011-gcp \\
\bottomrule
\end{tabular*}
\end{table}

\subsection{Software Versions}\label{sec:app_software}

\begin{table}[h]
\centering
\caption{Software versions (identical on both VMs).}\label{tab:software}
\begin{tabular*}{\textwidth}{@{\extracolsep{\fill}}lll@{}}
\toprule
Package     & Version & Role \\
\midrule
Rust        & 1.97.1  & TreeWalker + sweep\_bench \\
GCC         & 15.2.0  & Native libraries, tl2cgen, lleaves linking \\
LightGBM    & 4.7.0   & Model training + native predict \\
XGBoost     & 3.4.1   & Model training + native predict \\
treelite    & 4.7.2   & Model export (binary) + GTIL reference \\
tl2cgen     & 1.0.0   & Compiled predict baseline \\
lleaves     & 1.4.1   & LLVM-compiled predict baseline \\
llvmlite    & 0.50.0  & LLVM 22.1 for lleaves \\
quickscorer-rs & \texttt{8314d03} & QuickScorer baseline \\
\bottomrule
\end{tabular*}
\end{table}

\begin{table}[h]
\centering
\caption{Inference methods compared in \S\ref{sec:evaluation}. TreeWalker (full walk) uses TreeWalker's data layout but evaluates rows independently, isolating partial evaluation from implementation differences. Each baseline with two interfaces is timed through the faster one per cell, chosen by an untimed probe.}\label{tab:methods}
\small
\begin{tabularx}{\textwidth}{@{}l>{\raggedright\arraybackslash}X@{}}
\toprule
Method & Description \\
\midrule
LightGBM native                 & C API: \texttt{PredictForMat} or a loop of \texttt{PredictForMatSingleRowFast}; CMake Release with \texttt{-march=native} \\
XGBoost native                  & C API: one \texttt{PredictFromDense} call or a one-row loop; CMake Release with \texttt{-march=native} \\
tl2cgen (LightGBM+XGBoost)      & Compiled C trees via treelite (GCC \texttt{-O3 -march=native}); its runtime or a loop of the generated \texttt{predict} \\
lleaves (LightGBM)              & LLVM-compiled trees (\texttt{-O3} code generation for the host CPU; trees not inlined into the forest loop) \\
QuickScorer (LightGBM+XGBoost)  & Bitvector-based predictor (\texttt{quickscorer-rs}), on the encoded Expedia variant and the five original-Expedia cells of LightGBM's $T{=}50$, $L{=}2$ model (Appendix~\ref{sec:app_quickscorer}) \\
TreeWalker                      & Partial evaluation (this work) \\
TreeWalker (full walk)          & Row-independent baseline, same data layout \\
TreeWalker ($128$-row chunks)   & TreeWalker on $128$-row chunks of wider groups (Appendix~\ref{sec:app_chunked}) \\

\bottomrule
\end{tabularx}
\end{table}

Table~\ref{tab:software} lists all package versions.  TreeWalker is
compiled with \mbox{\texttt{RUSTFLAGS="-C~target-cpu=native"}}.  tl2cgen
libraries use GCC at \texttt{-O3} with \texttt{-march=native}; the
flags go into tl2cgen's compiler command, which ignores
\texttt{CFLAGS}.  lleaves compiles
through the LLVM that llvmlite bundles, with \texttt{fblocksize=34}
and code generation at \texttt{-O3} for the host CPU. lleaves's own
pipeline also runs LLVM's IR optimizer, which inlines every tree into
the forest loop; that optimizer does not finish on models of $500$
trees or more, so for every model we compile the trees in four
separate modules without it, and each tree remains a function call per
row.  LightGBM and XGBoost
native libraries are built from pinned sources with CMake Release and
\texttt{-march=native}, and predict on one thread.  Every method runs
on one thread, receives f64 row-major input, and converts it inside its
timed call.

\subsection{Numerical Precision}\label{sec:app_precision}

Every precision measurement in this section ran on the benchmark machine
types (Table~\ref{tab:hardware_full}). TreeWalker's and native XGBoost's
values were computed on each machine; the references (GTIL's outputs
and the per-row stage oracle) were computed once, on the Intel machine
that prepared the cells. TreeWalker adds each row's
leaf values in f64, in its loaded tree order, and applies the model's
base score once after the sum.

For f64 models (LightGBM), TreeWalker is bit-exact against treelite's
General Tree Inference Library (GTIL) when tree ordering is disabled:
$0$~ULP on all $30.2$M rows of the $356$ LightGBM cells, on both
machines. Tree ordering, on by default, permutes the trees and so the
order of the additions. It moves raw margins by at most
$1.9 \times 10^{-13}$ on Expedia and on survival panels of up to $256$
rows, $2.0 \times 10^{-10}$ on FLCHAIN's $512$-row panels and
$6.3 \times 10^{-9}$ on its $1{,}008$-row panels (at $T{=}2000$,
$L{=}16$); outputs move by at most $7.4 \times 10^{-15}$,
$9.3 \times 10^{-15}$ and $2.1 \times 10^{-13}$ respectively. The shifts are identical on Intel
and Arm, as f64 additions in a fixed order must be.

\paragraph{f32 accumulation roundoff.}
For f32 models (XGBoost), TreeWalker compares inputs and thresholds in
f32, so it takes the native split decisions, and accumulates the f32
leaf values in f64 rather than f32. Against an exactly rounded f64
reference (the correctly rounded sum of GTIL's per-tree outputs, with
the same base score), TreeWalker's raw margins are closer than native
XGBoost's on all $117{,}096$ audited rows of the $292$ XGBoost models (one cell each),
with no ties, on both machines; TreeWalker differs from the reference
on $7$ rows, by at most $8.9 \times 10^{-16}$. Native XGBoost differs on
every row, by at most $3.95 \times 10^{-5}$, and equals the reference
rounded to f32 on $7.8\%$ of them. The improvement is a roundoff
property of f64 versus f32 accumulation, not a model-quality change.

\subsection{Blocked Paired Measurement Protocol}\label{sec:app_methodology}

All methods run inside a single Rust binary via C~FFI, pinned to core
$0$ (\texttt{taskset -c 0}), timed with the same tick counter
(\texttt{rdtsc} at $2.3$\,GHz on Intel, \texttt{cntvct\_el0} at
$1$\,GHz on Arm), with no Python in the measurement path. Each cell runs in a fresh process, so
no method inherits a heap that earlier models churned. Per cell:

\begin{description}
\item[Validation.]
Before timing, every method's outputs on the timed groups are checked
against TreeWalker's production predict, within the model's rounding
bound (at least $10^{-13}$, or $10^{-5}$ for a method whose output
passes through f32); TreeWalker also matches its own full walk bit for
bit and agrees within $10^{-12}$ (relative) with a correctly rounded
per-row reference (the stage oracle) on a subset of each cell's groups. tl2cgen's known defect is checked apart
(Appendix~\ref{sec:app_missing}). A method that fails is not timed.

\item[Interface probe.]
Each baseline with two interfaces is timed through the faster one
(Table~\ref{tab:methods}), chosen by an untimed run on the warm-up
batch. One-row groups use the single-row interface.

\item[Sampling.]
A cell whose groups exceed $65{,}536$ rows per round times a seeded
sample of them, the same for every method, stratified by group size
and giving every group the same inclusion probability, with at least
$200$ groups. The SUPPORT $104$-row cells time $630$--$631$ of $1{,}774$
patients, Expedia's sessions $2{,}641$--$2{,}642$ of $10{,}000$, and FLCHAIN's
$1{,}008$-row cells $200$ of $1{,}304$. Medians are unweighted, and
inference is for the cell's pool.

\item[Rounds and blocks.]
The timed groups are split into $12$ batches. A round gives each
method its own phase: a warm-up on the batch it times last, then a
timed pass over every batch, with the same groups in the same order
for every method. Method order alternates across rounds, forward then
reversed. A \emph{block} is one batch in one round; its value is the
median time per group in it.

\item[Estimand.]
A method's latency is the median over its blocks, each weighted
equally. A speedup is the ratio of two latencies. Its $95\%$ interval
comes from $1{,}000$ bootstrap resamples that draw groups and rounds as
crossed clusters, pair the two methods, and recompute the medians.
p5--p95 over blocks describes dispersion and is not a confidence
interval.

\item[Stopping and re-timing.]
A cell times at least four rounds, two forward-and-reversed cycles, and
stops when every method's per-round mean time per row has a relative
standard error below $1\%$, or at $10$ rounds or its time budget. A cell
and mode whose speedup or ablation-effect interval is wider than
$\pm 2\%$ is timed once more with twice the rounds, on the same VM; the
longer attempt is reported, and both stay in the data.

\item[XGBoost processes.]
On Intel, XGBoost's batch walk runs at one of two speeds from process
to process, so each XGBoost cell also times XGBoost alone
in six extra one-round processes. The reported number is the first
process, timed and paired like every method; the faster mode across
processes is a sensitivity check, against which TreeWalker is
$5.3$--$9.0\times$ faster at the reference (row-weighted), and the
extra processes validate their first $32$ groups.

\item[Sentinel.]
A fixed cell (SUPPORT, LightGBM, $T{=}500$, $L{=}4$, $G{=}16$) runs at the start
of each run and after every $25$ cells, and a method whose latency
drifts beyond $2\%$ from the first measurement is flagged.
\end{description}

\paragraph{Cross-architecture reproducibility.}
Models are trained once on the Intel machine (seed 42) and the full
artifact set (native models, treelite exports, test data, group offsets)
is shared with the Arm machine, which checks every file's hash.
Compiled baselines and \texttt{sweep\_bench} are built per architecture
from the shared model files. Work counters are identical on both
machines.

\subsection{Coverage}\label{sec:app_missing}

All $712$ cells completed on both architectures, every required method
validated, and no cell failed. tl2cgen reads each feature through a
union of an integer and a double and tests the integer for its
missing-value marker, so in an f64 model a finite value whose low $32$
bits are all ones is read as missing. Such values occur in $68$ of the
$356$ LightGBM cells; their $2{,}673$ affected rows are checked apart,
$2{,}579$ of them differ, by at most $0.88$, and tl2cgen stays timed as
shipped. QuickScorer has no categorical splits or
missing-value routing and its implementation limits trees to $128$
leaves, so it is excluded up front from every LightGBM model with a
categorical split, every cell with missing values a model may route
right, and every model with a larger tree. It validates on the five
original-Expedia cells of LightGBM's $T{=}50$, $L{=}2$ model and on
$11$ of the $16$ cells of the encoded Expedia variant; the other $5$
fail validation and are not timed (Appendix~\ref{sec:app_quickscorer}).

\subsection{Measurement Variability}\label{sec:app_variability}

Of the (cell, mode) pairs, $97\%$ on Intel and $96\%$ on Arm stopped at
the $1\%$ precision target and the rest at the round cap or, for two
Arm pairs, the time budget. The $95\%$ interval's half-width is
$0.87\%$ of the speedup at the median (cell, method) pair on Intel
($0.98\%$ on Arm) and about $2\%$ at the $95$th percentile. The
re-timing rule re-timed $155$ (cell, mode) pairs on Intel and $164$ on
Arm, and every attempt ran on its parent's VM. In the grid $94$
(Intel) and $71$ (Arm) pairs stay wider than $\pm 2\%$ after their
re-time; in all but $2$ ($6$) of them the widest interval is a
baseline's speedup, most often tl2cgen's, which reaches $\pm 38\%$ on FLCHAIN's
$1{,}008$-row panels on Intel ($\pm 22\%$ at most on Arm). No method's sentinel latency drifted by
more than $0.6\%$. The two ablation runs agree within $0.9$ percentage
points on Intel and $2.5$ on Arm. As a sensitivity check, the
row-weighted speedup (total time over total time) is within $0.5\%$ of
the latency ratio at the median panel or cohort cell, and $4$--$6\%$
lower on Expedia sessions. Figure~\ref{fig:method_comparison} shows
p5--p95 over blocks as error bars.

\section{Implementation Details}\label{sec:app_implementation}\label{sec:app_layout}

\subsection{Data Layout}

The algorithm's constant factors depend on the node representation and
memory layout. TreeWalker uses a compact 16-byte node
(Table~\ref{tab:node}). The \texttt{value} field stores the split
threshold, an inline categorical bitset or an offset into the bitset
pool, or the leaf output, depending on node type. The varying-feature type (constant, monotonic, or
non-monotonic) is embedded in the \texttt{flags} byte at parse time,
avoiding per-node lookups during the hot loop. A row mask is one
\texttt{u16}, \texttt{u32} or \texttt{u64} word for groups of up to
$64$ rows and an array of 64-bit words above that.

\begin{table}[!htb]
\centering
\caption{Node layout (16 bytes).}\label{tab:node}
\begin{tabularx}{\textwidth}{@{}llX@{}}
\toprule
Field & Type & Description \\
\midrule
value       & f64  & Threshold, inline category bitset or bitset-pool offset, or leaf output \\
skip        & i16  & Tree-local index of the light child ($-1$ = leaf) \\
varying\_pred\_id & u16  & Index into varying-predicate array \\
feature     & u16  & Feature index \\
flags       & u8   & Bitfield: default\_left, categorical, inline\_cat, heavy\_is\_left, walkable, vary\_type (2 bits: constant, increasing, decreasing, non-monotonic) \\
cat\_n\_words & u8   & Bitset word count (categorical splits) \\
\bottomrule
\end{tabularx}
\end{table}

All trees share a single contiguous node array, reordered at parse time
via heavy-path DFS: the child with more training weight (LightGBM's
sample count, XGBoost's hessian sum; the left child on a tie) occupies
position \texttt{idx\,+\,1} in memory, making the constant-walk
fall-through path a sequential memory scan.

\subsection{Engineered Variant of Algorithm~\ref{alg:partial_eval}}

The conceptual algorithm in \S\ref{sec:partialeval} treats both
children of every varying split as recursive calls. The implementation
augments this with four optimizations (named in \S\ref{sec:tuning});
the engineered variant is shown in
Algorithm~\ref{alg:partial_eval_impl}.

\begin{algorithm}[!htb]
\caption{Engineered variant of Algorithm~\ref{alg:partial_eval}: adds
the unsplit shortcut, heavy-path DFS recursion, and tail-call
elimination. Constant walk on lines 9--12; leaf
accumulation on lines 3--7; varying partition on lines 15--22.}\label{alg:partial_eval_impl}
\begin{algorithmic}[1]
\Procedure{PartialEvalImpl}{tree, node, mask}
\Loop \Comment{constant walk + tail-loop site}
    \If{node is leaf}
        \ForAll{bit $r$ set in mask}
            \State $\text{results}[r] \gets \text{results}[r] + \text{node.value}$
        \EndFor
        \State \Return
    \EndIf
    \If{node.feature is constant}
        \State node $\gets$ child matching constant value
        \State \textbf{continue}
    \EndIf
    \State \textbf{break} \Comment{hit a varying split}
\EndLoop
\State $(m_L, m_R) \gets \Call{Partition}{\text{node, mask}}$ \Comment{may use precomputed mask}
\If{$m_L = 0$} node $\gets$ right child; mask $\gets m_R$; \textbf{loop} \Comment{unsplit shortcut}
\ElsIf{$m_R = 0$} node $\gets$ left child; mask $\gets m_L$; \textbf{loop}
\Else
    \State $(\text{light}, \text{heavy}) \gets$ structurally light / heavy children \Comment{heavy-path DFS}
    \State \Call{PartialEvalImpl}{tree, light.node, light.mask}
    \State node $\gets$ heavy.node; mask $\gets$ heavy.mask; \textbf{loop} \Comment{tail-call}
\EndIf
\EndProcedure
\end{algorithmic}
\end{algorithm}

\paragraph{Unsplit shortcut.}
When a varying-split partition leaves one side empty, the shortcut
loops on the non-empty side without making a recursive call. This is
the mechanism behind Lemma~\ref{lem:unsplit}: TreeWalker makes exactly
$B(\tau, X)$ non-tail recursive calls under the unsplit-enabled
evaluator, eliminating the $U = V - B$ would-be empty descents.

\paragraph{Heavy-path DFS structural recursion.}
At parse time each split's structurally heavy child (the one with more
training weight) is placed at \texttt{idx + 1} in the contiguous node
array; the light child is reached via the \texttt{skip} pointer. The
recursion descends into the light child and tail-loops on the heavy
child, so the tail-loop side stays sequential in cache. The corresponding stack-depth bound is
$O(L)$ (Lemma~\ref{lem:stack}); we make no $O(\log_2 G)$ claim.

\paragraph{Monotonic early-termination.}
When a varying split uses a feature monotone along the row sequence
($r=0,\ldots,G-1$), the partition stops comparing thresholds at the
cutoff, since the rows before and after it go opposite ways; missing
values are still checked and follow the default direction.
It applies to the trace evaluator only, since the precompute evaluates
every predicate the same way. It is a constant-factor reduction in
$Q_{\mathrm{trace}}$ whose wall-clock effect changes sign across the
ablation anchors, $-3\%$ to $+7\%$ on Intel ($-4\%$ to $+7\%$ on Arm; \S\ref{sec:ablation}).

\paragraph{Predicate precomputation.}
At each varying split TreeWalker must evaluate the split predicate for
every active row. Without precomputation this costs $O(P \cdot G)$
per group, where $P$ is the total varying-predicate count across all
trees. We reduce this to the cost in Eq.~\eqref{eq:Q-pre},
$O(F_v \cdot G \log G + |\Pi^{\mathrm{num}}| + G \cdot |\Pi^{\mathrm{cat}}|)$,
by precomputing every varying predicate before traversal: at parse time, extract the
unique predicates per varying feature; at predict time, sort the $G$
row values once per feature and sweep a two-pointer merge against the
sorted thresholds, emitting one bitmask per predicate in amortized
$O(1)$. Tree traversal then looks up precomputed masks by predicate
index. This strategy is the $Q_{\mathrm{pre}}$ instantiation analyzed
in Corollary~\ref{cor:impl}.

\paragraph{Production kernel.}
Production predict inlines the first walk of each tree into the tree
loop; with prefix grouping, on by default with depth $2$, that walk
starts where the tree's group of shared prefixes ends. When the
constant walk reaches a leaf before any varying split, as in every tree
of the measured $G{=}1$ cells, whose models split on no varying
feature, the leaf's value is added to every row directly, with no
recursive call. In models of $4{,}096$ nodes ($64$\,KiB) or more, at
least the L1 data cache of either machine, the loop prefetches each
next tree's start node. Neither changes an output or a work
counter.

\section{Proof of Theorem~\ref{thm:decomp}}\label{sec:app_proof}

We prove four supporting lemmas and then close out the bound stated in
\S\ref{sec:theorem}. The setup, definitions
($A(u), C, V, T_c, Q$), cost model, and provisioning step are as in
\S\ref{sec:theorem}. We additionally use, only in Lemma~\ref{lem:unsplit},
the partition counts
\[
  B(\tau, X) = \bigl|\{u : u \text{ varying}, |A_L(u)|\!>\!0 \land |A_R(u)|\!>\!0\}\bigr|,
  \qquad
  U(\tau, X) = V(\tau, X) - B(\tau, X).
\]

\begin{lemma}[Active-mask disjointness]\label{lem:disjoint}
Under the default unsplit-enabled evaluator, for every node $u \in \tau$
with $A(u) \neq \emptyset$, TreeWalker visits $u$ exactly once during
the partial evaluation of $\tau$ on $X$.
\end{lemma}
\begin{proof}
By induction on the depth of $u$ in $\tau$. The root has
$A(\mathrm{root}) = [G]$ and is visited once at the call entry. At a
constant split $u$ with $A(u) \neq \emptyset$, the algorithm forwards
$A(u)$ unchanged to the deterministic child (the other child is
excluded from $T_c$). At a varying split $u$ with $A(u) \neq \emptyset$,
the algorithm partitions $A(u)$ into disjoint masks
$(A_L(u), A_R(u))$ with $|A_L| + |A_R| = |A(u)|$, and \emph{continues
once on each non-empty side}: one recursive call to one child plus a
tail-loop to the other if both are non-empty (Lemma~\ref{lem:unsplit}
for the unsplit-enabled call-count rule). The unsplit shortcut ensures
that empty-mask sub-calls are not made: a child with empty mask is
skipped, so no empty-mask visit occurs. By induction, every node
reached with non-empty active mask is visited exactly once.
\end{proof}

\noindent
\emph{Remark (ablation).} The unsplit-disabled ablation
(\texttt{disable\_unsplit=true}) makes a recursive call regardless of
whether the child mask is empty, so empty-mask entries are possible.
The theorem and lemmas in this section are stated for the default
evaluator; ablation-path counters lie above these bounds, not below.

\begin{lemma}[Visited-set bound]\label{lem:visit-bound}
For every group $X$, $C(\tau, X) + V(\tau, X) \le |T_c(\tau, X)|.$
\end{lemma}
\begin{proof}
Every split node visited by TreeWalker (with non-empty active mask) lies
in $T_c(\tau, X)$ by construction, since the constant projection retains
every branch a non-empty mask can reach. By
Lemma~\ref{lem:disjoint} each such visit is counted once, so $C + V$
is the count of internal split nodes of $T_c$ that have non-empty
masks under the partial evaluation. The internal split nodes of $T_c$
number $|T_c| - L_c(\tau, X)$, where $L_c$ is the count of $T_c$
leaves; the visited-with-non-empty-mask set is a subset of those
internal nodes, hence $C + V \le |T_c| - L_c \le |T_c|$.
\end{proof}

\noindent
\emph{Remark (tightness).} The bound is conservative. $T_c$ retains
both subtrees below every varying split, even when the partition sends
all active rows to one side. The looseness can be
$\Theta(|T_c|/L)$ in adversarial cases (e.g.\ a varying split at the
root with all rows going one way, leaving a large unvisited subtree on
the other side). The conservative form suffices for
Theorem~\ref{thm:decomp}; tighter bounds follow from $C(\tau, X) +
V(\tau, X)$ themselves.

\begin{lemma}[Unsplit call-count]\label{lem:unsplit}
Under the default unsplit-enabled evaluator, TreeWalker's partial
evaluation of $\tau$ on $X$ makes exactly $B(\tau, X)$ non-tail
recursive calls. The unsplit shortcut eliminates $U(\tau, X) = V(\tau,
X) - B(\tau, X)$ continuations into empty subtrees.
\end{lemma}
\begin{proof}
At a varying split $u$ with $A(u) \ne \emptyset$:
(i) if both children have non-empty masks ($u$ counted in $B$), the
algorithm recurses once on one child and tail-loops on the other,
producing exactly one non-tail recursive call;
(ii) if exactly one child has a non-empty mask ($u$ counted in $U$),
the unsplit shortcut tail-loops on the non-empty child and makes no
recursive call. The total non-tail call count is therefore $B(\tau,
X)$, and all $U$ would-be empty-side descents are eliminated.
Lemma~\ref{lem:unsplit} bounds call count, not row work; it is not
used in the proof of Theorem~\ref{thm:decomp} and is recorded here for
completeness (the unsplit ablation adds at least $U$ calls, \S\ref{sec:ablation}).
\end{proof}

\begin{lemma}[Stack depth]\label{lem:stack}
Stack depth during partial evaluation of $\tau$ is at most~$L$ (the
maximum tree depth).
\end{lemma}
\begin{proof}
Each recursive call descends to a child of the current node. The
non-recursive (tail-looped) side does not increase stack depth. Since
the height of $\tau$ is at most $L$, the depth of the recursion is at
most $L$.
\end{proof}

\noindent
\emph{Remark (heavy-path DFS).} A tighter $O(\log_2 G)$ stack-depth
bound would require the recursive call to always descend to the side
with smaller active popcount. TreeWalker instead recurses on the
\emph{structurally light} child (under heavy-path DFS, fixed at parse
time) so that the tail-loop side stays contiguous in cache. We make no
$O(\log_2 G)$ claim.

\paragraph{Proof of Theorem~\ref{thm:decomp}.}
Every visited split node contributes one unit of control work (one
visit, by Lemma~\ref{lem:disjoint}), giving $C(\tau, X) + V(\tau, X)$
in total. Every row reaches exactly one leaf per tree (the active-mask
partition is exact: $\sum_{\ell \in \mathrm{leaves}(\tau)} |A(\ell)| =
G$), contributing exactly $G$ leaf writes. Provisioning queries at
varying splits aggregate to $Q(\tau, X)$. Equality at unit cost
follows. The structural inequality $C + V \le |T_c|$ is
Lemma~\ref{lem:visit-bound}. \qed

\noindent
\emph{Remark.} Theorem~\ref{thm:decomp} is stated for the per-tree
evaluator. The default forest path includes a prefix-grouping pass
that begins per-tree evaluation at a shared continuation node rather
than the root. Prefix grouping only lowers the structural control term
at ensemble level by sharing visits across trees; the per-tree theorem
remains an upper bound on each summand of the ensemble cost.

\paragraph{Proof of Corollary~\ref{cor:trace}.}
Counting provisioning queries: each row visits each varying split on
its path exactly once. Summing per-row contributions gives
$Q_{\mathrm{trace}} = \sum_x d_{v,\tau}(x) = G \cdot d_v$ (an exact
identity, not an upper bound). Substituting Theorem~\ref{thm:decomp}
with $Q_{\mathrm{trace}} = G \cdot d_v$ and dividing by $G(d + 1)$
yields~\eqref{eq:trace-ratio}, since $|T_c| + G + G d_v = |T_c| +
G(d_v + 1)$. The asymptote follows from $|T_c|/(G(d+1)) \to 0$ as
$G \to \infty$ under the boundedness assumption; the leaf-inclusive
denominator avoids a spurious $1/d$ term. \qed

\paragraph{Proof of Corollary~\ref{cor:impl}.}
Per varying feature: sorting the $G$ group row values is $O(G \log G)$.
Numeric predicates use a two-pointer merge in $O(|\Pi_f^{\mathrm{num}}|)$
amortized over the sort. Categorical predicates evaluate per row in
$O(G \cdot |\Pi_f^{\mathrm{cat}}|)$. Summing over $f \in F_v$ gives
Eq.~\eqref{eq:Q-pre}. Each tree's per-visit provisioning access is
$O(1)$ (bitmask lookup), so the per-tree control term in
Theorem~\ref{thm:decomp} reduces to $C(\tau_i, X) + V(\tau_i, X) + G$.
Summing over trees and adding the shared provisioning term
$Q_{\mathrm{pre}}$ yields the bound. \qed

\paragraph{Empirical validation by counter.}
The decomposition~\eqref{eq:decomp} is an identity on every group, not
a predicted-vs-observed ratio. We validate componentwise against
TreeWalker's emitted counters:

\begin{table}[h]
\centering
\small\begin{tabular*}{\textwidth}{@{\extracolsep{\fill}}lll@{}}
\toprule
\textbf{Term} & \textbf{Counter (proxy)} & \textbf{Source} \\
\midrule
$C(\tau, X)$ & \texttt{constant\_steps} & per-tree stats (exact) \\
$V(\tau, X)$ & \texttt{varying\_splits} & per-tree stats (exact) \\
$G$ & group size & known (exact) \\
$Q_{\mathrm{pre}}$ & \texttt{precompute\_row\_evals} $+$ \texttt{partition\_row\_evals} & per-group stats (proxy) \\
$|T_c(\tau, X)|$ & offline replay against $(\tau, X)$ & one-time post-process \\
\bottomrule
\end{tabular*}
\caption{Trace counters used to validate Theorem~\ref{thm:decomp}
componentwise. $C$, $V$, $G$ are tracked exactly; $Q_{\mathrm{pre}}$
is a row-evaluation proxy that misses the $G\log G$ sort term and the
predicate-set-size terms in Eq.~\eqref{eq:Q-pre}.}
\label{tab:decomp-counters}
\end{table}

\section{QuickScorer Head-to-Head}\label{sec:app_quickscorer}

QuickScorer~\citep{lucchese2015quickscorer} represents each tree as a
bitvector of leaves and intersects per-feature interval-bitvectors
across all trees to short-circuit traversal. The only available
open-source implementation~\citep{quickscorer_rs} caps trees at
$\le 128$ leaves, an implementation choice rather than a limit of the
QuickScorer family. It also reads no categorical splits, sends every
missing value left, and compares thresholds in f32. On our datasets
that is a different function: the survival models split on categorical
features, and Expedia has missing values. The exception is Expedia's
LightGBM model at $T{=}50$, $L{=}2$, which has no categorical split and
sends no missing value right: on its five cells (sessions and cohorts
of $4$--$32$ rows) QuickScorer validates and is faster than TreeWalker
on both machines, by $1.3$--$2.5\times$ on Intel and $1.1$--$2.1\times$
on Arm.

We therefore compare on an encoded Expedia variant, with every missing
value replaced before training by a constant below its column's
training minimum, the same in train and test, so that QuickScorer's
left-going default never applies. The encoding does not make
QuickScorer compute LightGBM's function, because it still compares
thresholds in f32: $5$ of the $16$ cells ($T \in \{50, 500, 1000,
2000\}$, $L \in \{2, 4\}$, both frameworks) fail validation and are not
timed. Table~\ref{tab:quickscorer_h2h} reports the other $11$. QuickScorer
is faster than TreeWalker in all $11$ on Intel, by up to $1.75\times$ on
the smallest models, and in $5$ on Arm, where TreeWalker is faster in
the $L{=}4$ cells from $T{=}500$ on and at $T{=}2000$, $L{=}2$
($1.03$--$1.72\times$). QuickScorer's f64-to-f32 input
conversion is inside its timed call, as every method's conversion is.

\begin{table}[h]
\centering
\small
\caption{QuickScorer vs.\ TreeWalker on the encoded Expedia variant, in
the $11$ cells whose QuickScorer outputs validate. Latencies are Intel's,
in $\mu$s per session. QS/TW is QuickScorer's latency divided by
TreeWalker's; values below $1$ mean QuickScorer is faster. The two
$0.98$ ratios on Intel are within $\pm 2\%$.}
\label{tab:quickscorer_h2h}
\begin{tabular*}{\textwidth}{@{\extracolsep{\fill}}lrrrrrr@{}}
\toprule
Framework & $T$ & $L$ & TreeWalker ($\mu$s) & QuickScorer ($\mu$s) & QS/TW (Intel) & QS/TW (Arm) \\
\midrule
LightGBM & 50   & 2 & 5.38   & 3.44   & $0.64$ & $0.75$ \\
LightGBM & 500  & 4 & 67.13  & 61.22  & $0.91$ & $1.43$ \\
LightGBM & 1000 & 4 & 107.35 & 105.35 & $0.98$ & $1.61$ \\
XGBoost  & 50   & 2 & 6.85   & 3.90   & $0.57$ & $0.69$ \\
XGBoost  & 50   & 4 & 16.95  & 10.35  & $0.61$ & $0.88$ \\
XGBoost  & 500  & 2 & 34.85  & 24.09  & $0.69$ & $0.77$ \\
XGBoost  & 500  & 4 & 72.67  & 65.01  & $0.89$ & $1.45$ \\
XGBoost  & 1000 & 2 & 55.77  & 45.29  & $0.81$ & $0.89$ \\
XGBoost  & 1000 & 4 & 116.57 & 111.23 & $0.95$ & $1.60$ \\
XGBoost  & 2000 & 2 & 91.41  & 82.67  & $0.90$ & $1.03$ \\
XGBoost  & 2000 & 4 & 201.93 & 198.71 & $0.98$ & $1.72$ \\
\bottomrule
\end{tabular*}
\end{table}

The two methods optimize complementary axes (per-walk vs.\
number-of-walks) and could in principle compose: a partial-evaluation
outer loop with a QuickScorer-style inner kernel for the varying
subtrees. We have not implemented this composition.

\section{Supplementary Figures}\label{sec:app_sensitivity}

\begin{figure}[H]
\centering
\includegraphics[width=0.85\textwidth]{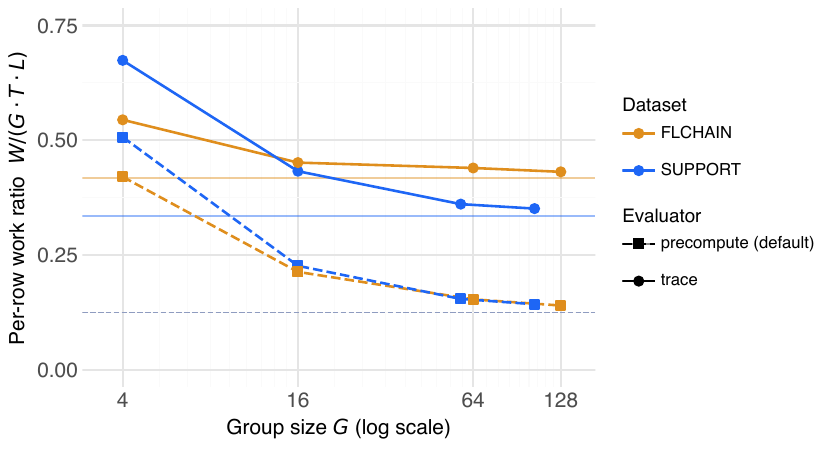}
\caption{Empirical convergence of the per-row work ratio
$W(\mathcal{F}, X) / (G \cdot T \cdot L)$ vs.\ group size $G$
(log scale) at the reference configuration ($T=500$, $L=8$, LightGBM)
on the fixed-$G$ datasets; the counts are the same on both machines. Solid lines: trace provisioning (no
precompute). Dashed lines: precompute provisioning (default).
Horizontal reference lines mark the limits under this baseline:
the Corollary~\ref{cor:trace} asymptote $(d_v+1)/L$ for trace
(per-dataset, in matching color) and the structural leaf-write floor
$1/L$ for precompute (dataset-independent, the lowest line at
$0.125$). Trace approaches its limit from above as row-dependent
predicate work scales with $G$; precompute drops faster because its
per-group provisioning cost is spread across more rows. Visualizes the data in Table~\ref{tab:decomp-results} (\S\ref{sec:decomp_validation}).}
\label{fig:asymptotic-convergence}
\end{figure}

\begin{figure}[H]
\centering
\includegraphics[width=\textwidth]{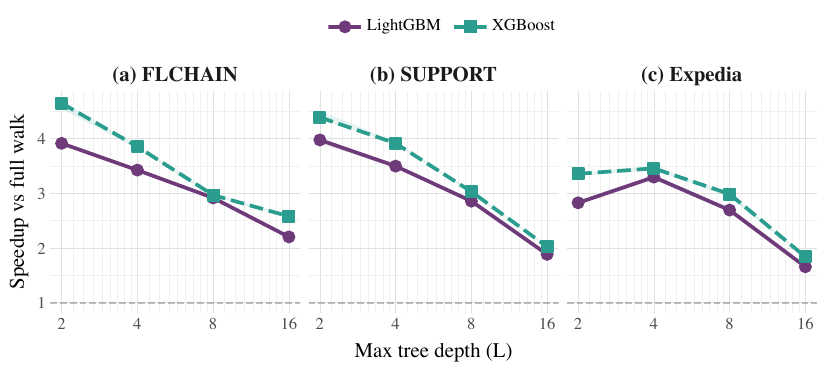}
\caption{Algorithmic speedup vs.\ tree depth $L$ ($T{=}500$, $G{=}16$, Intel).  On the survival datasets shallower trees yield higher speedup.  Ribbons: $95\%$ intervals.}\label{fig:depth_sensitivity}
\end{figure}

\begin{figure}[H]
\centering
\includegraphics[width=\textwidth]{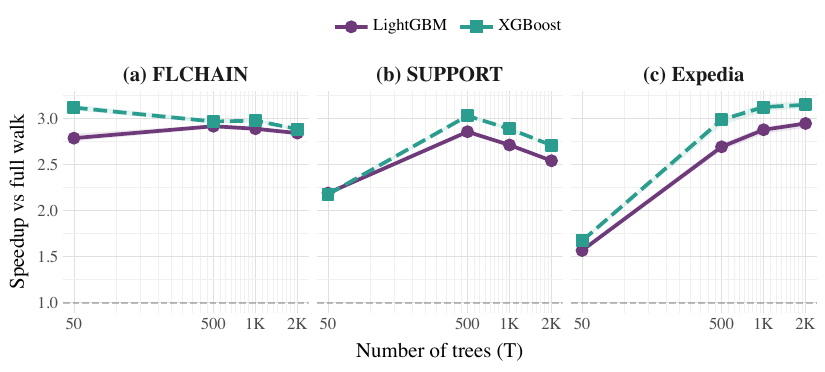}
\caption{Algorithmic speedup vs.\ ensemble size $T$ ($L{=}8$, $G{=}16$, Intel).  Ribbons: $95\%$ intervals.}\label{fig:ntrees_scaling}
\end{figure}

\begin{figure}[H]
\centering
\includegraphics[width=0.6\textwidth]{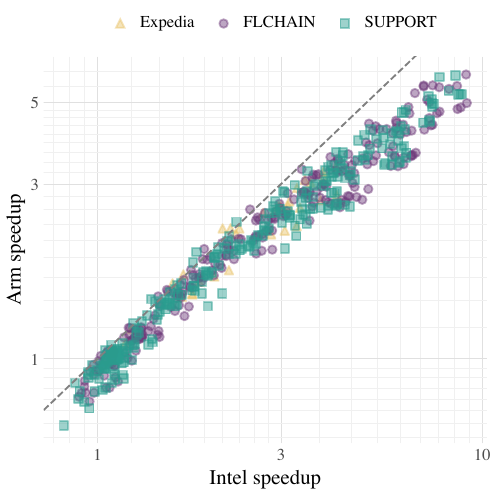}
\caption{Intel vs.\ Arm algorithmic speedup for the $544$ cells of the factorial grid (log-log).  Points above the diagonal indicate higher Arm speedup; dashed line = parity.}\label{fig:cross_platform}
\end{figure}

\begin{figure}[H]
\centering
\includegraphics[width=\textwidth]{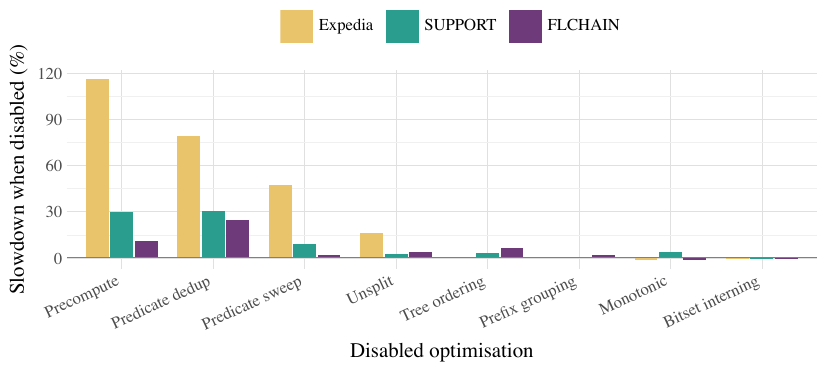}
\caption{Slowdown when each flag is disabled in isolation, in percent: the disabled latency over a research build of the kernel with every flag on, minus one (monotonic early-termination: against the build with precomputation off), at $T{=}500$, $L{=}8$ on Intel.  The research build is within $-4\%$ to $+6\%$ of production at the anchors.  $G$ is the reference $G{=}16$ for SUPPORT, empirical sessions for Expedia, and $G{=}4$ for FLCHAIN.  Flags are sorted by the geometric mean slowdown over the three anchors.  Bars: mean over the frameworks and both ablation runs, which agree within $0.9$ percentage points.}\label{fig:ablation_waterfall}
\end{figure}

\section{Scenario-Analysis Benchmark}\label{sec:app_scenario}

\paragraph{Workload.}
The dataset is UCI Default of Credit Card Clients~\citep{yeh2009credit}
(CC BY 4.0, OpenML 42477): $30{,}000$ rows, $23$ numeric features, binary
target. We train one LightGBM model ($T{=}500$, $L{=}8$,
\texttt{num\_leaves}$=255$, binary objective, seed $42$) on a
deterministic train split and draw $2{,}000$ base rows from the test
split. Each base row expands into $G$ variants; variant $0$ is the
unperturbed row. The perturbable pool is credit limit (\texttt{x1}),
bill amounts (\texttt{x12}--\texttt{x17}), and payment amounts
(\texttt{x18}--\texttt{x23}); the nested $k$ sets are ordered by
descending gain importance within the pool. Multipliers are
log-uniform in $[0.5, 2]$ (\texttt{np.random.default\_rng(42)}). All
non-selected features are bit-identical across the group.

\paragraph{Zero-valued features.}
A selected feature with base value $0$ stays $0$ under a multiplicative
shock, which is common for the payment amounts. At $k{=}8$,
$47.15\%$ of groups contain at least one selected zero-valued feature,
$15.00\%$ of selected entries are zero, and the mean realized varying
dimensionality is $6.80$ of $8$. We therefore measure $d_v$ on the
realized groups via the trace evaluator rather than deriving it from
$k$.

\paragraph{Protocol and correctness.}
Timing follows the protocol of \S\ref{sec:app_methodology} (at least
four rounds, $1\%$ precision target), in the same campaign and on the
same VMs as the main grid. TreeWalker and its full walk are validated
as there, against production predict within the model's rounding bound
(at least $10^{-13}$) and against the per-row reference. Models are
trained once on Intel and shared, so work counters are identical across
architectures. Latencies are in $\mu$s per scenario set (one base row
and its $G$ variants), not per variant;
Table~\ref{tab:scenario_full} reports unit-free speedups and ratios.

\begin{table}[h]
\centering
\small
\caption{Full scenario-analysis sweep. Speedup is algorithmic speedup over the same-layout full walk. Work ratios are $W/(G \cdot T \cdot L)$ under the trace and default (precompute) evaluators; the trace prediction under this baseline is $(d_v+1)/L$ with $L = 8$.}\label{tab:scenario_full}
\begin{tabular*}{\textwidth}{@{\extracolsep{\fill}}rrrrrrr@{}}
\toprule
$k$ & $G$ & Speedup (Arm) & Speedup (Intel) & Trace ratio & Predicted & Precompute ratio \\
\midrule
1 & 4   & $1.77\times$ & $1.96\times$ & 0.413 & 0.188 & 0.351 \\
1 & 16  & $4.46\times$ & $5.16\times$ & 0.245 & 0.188 & 0.183 \\
1 & 64  & $7.74\times$ & $11.79\times$ & 0.202 & 0.188 & 0.140 \\
1 & 128 & $9.21\times$ & $15.15\times$ & 0.195 & 0.188 & 0.133 \\
\midrule
2 & 4   & $1.47\times$ & $1.60\times$ & 0.502 & 0.268 & 0.359 \\
2 & 16  & $3.72\times$ & $4.16\times$ & 0.329 & 0.268 & 0.186 \\
2 & 64  & $6.90\times$ & $10.00\times$ & 0.284 & 0.268 & 0.141 \\
2 & 128 & $8.38\times$ & $12.82\times$ & 0.276 & 0.268 & 0.133 \\
\midrule
4 & 4   & $1.21\times$ & $1.29\times$ & 0.635 & 0.386 & 0.375 \\
4 & 16  & $2.93\times$ & $3.21\times$ & 0.453 & 0.386 & 0.193 \\
4 & 64  & $5.97\times$ & $8.01\times$ & 0.403 & 0.386 & 0.143 \\
4 & 128 & $7.37\times$ & $10.32\times$ & 0.395 & 0.386 & 0.135 \\
\midrule
8 & 4   & $1.04\times$ & $1.08\times$ & 0.875 & 0.601 & 0.401 \\
8 & 16  & $2.40\times$ & $2.57\times$ & 0.679 & 0.601 & 0.205 \\
8 & 64  & $4.97\times$ & $6.17\times$ & 0.622 & 0.601 & 0.148 \\
8 & 128 & $6.22\times$ & $8.12\times$ & 0.611 & 0.601 & 0.138 \\
\bottomrule
\end{tabular*}
\end{table}

\section{Groups Wider than 128 Rows}\label{sec:app_chunked}

\paragraph{Per-chunk decomposition.}
A group $X$ of width $G > 128$ splits into $K = \lceil G/128 \rceil$
chunks $X_1, \ldots, X_K$. Applying Theorem~\ref{thm:decomp} per chunk
bounds per-tree work by
\[
  W(\tau, X) \;\le\; \sum_{j=1}^{K} |T_c(\tau, X_j)| \;+\; G \;+\;
  \sum_{j=1}^{K} Q(\tau, X_j).
\]
The constant projection depends only on the tree and the constant
feature values, which all chunks share, so $|T_c(\tau, X_j)| =
|T_c(\tau, X)|$ for every $j$. The structural term is paid $K$ times
instead of once, and the per-row structural share plateaus at the
chunk width instead of falling further in $G$. Leaf writes remain $G$
in total, and $Q_{\mathrm{trace}}$ is additive over rows, so chunking
leaves both unchanged. For $Q_{\mathrm{pre}}$ the sort term does not
grow ($\sum_j G_j \log G_j \le G \log G$), while the predicate-sweep
term is paid once per chunk. This $K$-fold repayment of $|T_c|$ is the
analytical cost of chunking. Evaluating the group whole avoids it, at
the price of masks of $\lceil G/64 \rceil$ words, whose operations cost
$\Theta(G/64)$ (\S\ref{sec:theorem}); the engine does so for groups of
any width, in pieces of at most $1{,}024$ rows.

\paragraph{Chunks against whole groups.}
Because the engine takes wide groups whole, the cost of chunking is
measurable. On FLCHAIN's horizons of $256$, $512$ and $1{,}024$
($1{,}008$ rows) we time TreeWalker on whole groups and, in the same
cells and rounds, on $128$-row chunks of the same groups
($T \in \{500, 2000\}$,
$L \in \{8, 16\}$, both frameworks; $200$ of the $1{,}304$ patients at
$1{,}008$ rows). Chunks are slower in every cell, by $1.07$--$1.38\times$
on Intel and $1.07$--$1.31\times$ on Arm (Table~\ref{tab:chunked}). The
measurement is a net difference: chunks repay $|T_c|$ and the predicate
sweep once per chunk, while whole groups pay for wider masks. It does
not isolate the $K$-fold $|T_c|$ term.

\begin{table}[h]
\centering
\small
\caption{$128$-row chunks against whole groups on FLCHAIN: chunked
latency over whole-group latency, median [min, max] over the $8$ cells
($T \in \{500, 2000\}$, $L \in \{8, 16\}$, LightGBM and XGBoost) at
each horizon.}\label{tab:chunked}
\begin{tabular*}{\textwidth}{@{\extracolsep{\fill}}rrr@{}}
\toprule
Rows per group & Intel & Arm \\
\midrule
$256$      & $1.17\times$ [$1.13$, $1.21$] & $1.19\times$ [$1.14$, $1.26$] \\
$512$      & $1.22\times$ [$1.15$, $1.28$] & $1.20\times$ [$1.13$, $1.25$] \\
$1{,}008$  & $1.19\times$ [$1.07$, $1.38$] & $1.12\times$ [$1.07$, $1.31$] \\
\bottomrule
\end{tabular*}
\end{table}

\end{document}